\documentclass[journal]{IEEEtran}
\usepackage{times}
\usepackage{multicol}

\usepackage{graphicx}
\usepackage{subfigure}
\usepackage{epsfig}
\usepackage{amsmath,amssymb}
\usepackage{algorithmic}
\usepackage{url}
\usepackage{color}
\usepackage[ruled]{algorithm2e}
\usepackage{ntheorem}
\usepackage{float}
\usepackage{physics}
\usepackage{cite}
\usepackage{todonotes}
\usepackage{mathtools}

\newtheorem{thm}{Theorem}
\newtheorem{cor}{Corollary}
\newtheorem{lem}{Lemma}
\newtheorem{rem}{Remark}
\newtheorem{problem}{Problem}

\newcommand{\revo}[1]{{\color{black}{#1}}}
\newcommand{\rev}[1]{{\color{black}{#1}}}

\begin{document}

% paper title
\title{Sensor Assignment Algorithms to Improve Observability while Tracking Targets}

% You will get a Paper-ID when submitting a pdf file to the conference system
\author{Lifeng~Zhou,~\emph{Student Member,~IEEE,} and~Pratap~Tokekar,~\emph{Member,~IEEE}}

\maketitle

\begin{abstract}
We study two sensor assignment problems for multi-target tracking with the goal of improving the observability of the underlying estimator. We consider various measures of the observability matrix as the assignment value function. We first study the general version where the sensors must form teams to track individual targets. If the value function is monotonically increasing and submodular then a greedy algorithm yields a $1/2$--approximation. We then study a restricted version where exactly two sensors must be assigned to each target. We present a $1/3$--approximation algorithm for this problem which holds for arbitrary value functions (not necessarily submodular or monotone). In addition to approximation algorithms, we also present various properties of observability measures. We show that the inverse of the condition number of the observability matrix is neither monotone nor submodular, but present other measures which are. \revo{Specifically, we show that the trace and rank of the symmetric observability matrix are monotone and submodular and the log determinant of the symmetric observability matrix is monotone and submodular when the matrix is non-singular.} If the target's motion model is not known, the inverse cannot be computed exactly. Instead, we present a lower bound for distance sensors. In addition to theoretical results, we evaluate our results empirically through simulations.
\end{abstract}

\IEEEpeerreviewmaketitle

\section{Introduction}
Assigning sensors to better track targets is a well-studied \rev{problem~\cite{grocholsky2002information,jiang2003optimal,spletzer2003dynamic,williams2007information,kamath2007triangulation,zavlanos2008dynamic,tekdas2010sensor,nam2015your,tzoumas2016sensor,tzoumas2016near,tzoumas2017scheduling,yang2017algorithm,chopra2017distributed}.} The typical setup is to formulate a one-to-one assignment problem which can be solved using bipartite graph matching algorithms~\cite{kuhn1955hungarian}. Unlike these works, we focus on scenarios where multiple sensors can be assigned to each target. Furthermore, the utility of assigning multiple sensors may not even be a simple addition of individual utilities, and can have diminishing returns. We study two versions of these many-to-one assignment problems and give constant-factor approximation algorithms for each.

We study these assignment problems in settings where assigning the right set of sensors can improve the observability in tracking potentially mobile targets using noisy sensors. The performance of the target state estimators can be improved by exploiting the \emph{observability} of the underlying system~\cite{gadre2004toward,papadopoulos2010cooperative,arrichiello2013observability,williams2015observability}. Our formulation is motivated by scenarios where sensing multiple targets \revo{by the same sensor} can be time-consuming \revo{as is the case in the following applications}. For example, Tokekar et al.~\cite{tokekar2011active} presented a system to track radio-tagged fish in which each radio tag is assigned a unique frequency. In order to get a measurement for one frequency, the sensor spends two seconds listening to the periodic emissions. Therefore, tracking multiple targets by the same sensor can be time-consuming. Motivated by these scenarios, we seek to assign sensors to track targets with the constraint that each sensor is assigned at most one target. However, multiple sensors can be assigned to track the same target. In fact, more sensors can often improve the tracking performance. 

We first study a general assignment problem where there is no restriction on the number of sensors assigned to a target. We let the algorithm decide the optimal configuration of sensor teams assigned to each target. If the value function is submodular and monotone,\footnote{We use ``monotone" and ``monotone increasing" interchangeably.} a greedy algorithm gives a $1/2$--approximation~\cite{fisher1978analysis}. \revo{In general, this bound is tighter, i.e., $1-1/e$ by using the randomized continuous greedy algorithm~\cite{vondrak2008optimal}. However, for a deterministic algorithm, $1/2$--approximation is the best-known result~\cite{vondrak2008optimal}.} We show that observability measures such as the trace and rank of the symmetric observability matrix are submodular and monotone and the log determinant is submodular and monotone when the symmetric observability matrix is non-singular. However, the inverse of the condition number, a commonly-used measure, is neither monotone nor submodular\revo{, which is shown by a counterexample.} 

We then study a \emph{restricted} version of the problem where the value function can be arbitrary (not necessarily monotone and submodular) but exactly two sensors must be assigned to a target. \revo{In fact, we show in \revo{ Corollaries \ref{cor:sensor_n_1} and \ref{cor:sensor_n_gre2} that at least two range sensors must be assigned for the inverse condition number to be greater than zero.} Even for the case of bearing sensors, it is known that two sensors are necessary; there exists work on assignment and placement of pairs of bearing sensors for improving target tracking~\cite{kamath2007triangulation,tekdas2010sensor}. We study the analogous problem for range sensors.} We prove that the greedy algorithm achieves a $1/3$--approximation of the optimal solution \revo{for assigning pairs of range sensors}. 

Observability is a basic concept in control theory and has been widely applied in robotics. Observability for range-only beacon sensors, in particular, has been closely studied for underwater navigation. Gadre and Stilwell~\cite{gadre2004toward} analyzed the local and global observability~\cite{hermann1977nonlinear} for the localization of an autonomous underwater vehicle by an acoustic beacon. The problems of single vehicle localization and multi-vehicle relative localization are studied in \cite{arrichiello2013observability} using an observability criterion introduced in \cite{krener2009measures}. In these works, it is the sensors that are moving. Consequently, the sensors know their control vectors and can thus, compute the observability matrix and its measures. \revo{When tracking a target, however, the control inputs for the targets may be unknown.} In recent work, Williams and Sukhatme~\cite{williams2015observability} studied a multi-sensor localization and target tracking problem where they showed how to leverage graph rigidity to improve the observability for sensor team localization and robust target tracking. 

Unlike these works, we consider scenarios where the control inputs for the targets are not known to the sensors. Consequently, we cannot compute the observability matrix exactly. Instead, we present a novel lower bound on the observability for the case of unknown target motion tracked by range-only sensors. Specifically, we show how to lower bound the \emph{condition number}~\cite{krener2009measures} of the partially known observability matrix using only the known part (Section~\ref{sec:lowerbound}). \revo{The result is specific to the problem at hand where the inputs
appear, linearly, on a single line of the observability matrix.}

In addition to theoretical results, we conduct simulations to evaluate the empirical performance of the algorithms. We find that sensors are assigned to targets almost uniformly using the greedy algorithm for the first problem (Section~\ref{sec:simulation}-A). The greedy algorithm for the second problem performs much better than $1/3$ in practice (Section~\ref{sec:simulation}-B). 

\section{Overview of the Problem and Results} \label{sec:pro_formulation}
\revo{In this section, we first formally define the problems that are studied and then summarize the main contributions of this work.}

\subsection{Problem Formulation}
We consider a scenario where there are $N$ sensors and $L$ targets in the environment.  We use  $\sigma(t_l)$ to represent the set of sensors assigned to target $t_l$. $\sigma^{-1}(s_i)$ gives the set of targets assigned to sensor $s_i$. We use $\sigma_i(t_l)$ to give the $i^{th}$ sensor assigned to target $t_l$. We order the assigned sensors by using their IDs such that $\sigma_1(t_l) < \sigma_2(t_l) < \sigma_3(t_l) < \ldots$.  Let $\omega(s_i, s_j, t_l)$, and $\omega(\mathcal{S}_{i}, t_l)$ be some measure of the observability of tracking $t_l$ with $s_i$ and $s_j$, and with a set of sensor(s) $\mathcal{S}_{i}$, respectively. \revo{We calculate the observability measure by taking account of the relative positions of sensors and targets (which will appear in the observability matrix) at every time step}. 

We start with the problem of assigning a set of sensors to each target. The sensors form teams of varying sizes to track individual targets. Sensors within a team can share measurements so as to better track the targets. We constrain each sensor to be assigned to only one target. This is motivated by scenarios where sensing multiple targets can be time consuming (as is the case with radio sensors~\cite{tokekar2011active,alvarez2018acoustic} or communicating multiple measurements with the team can be time and energy consuming.
\begin{problem}[General Assignment] Given a set of sensors, $\mathcal{S} := \{s_0,\ldots,s_N\}$ and a set of targets, $\mathcal{T} := \{t_0,\ldots,t_L\}$, find an assignment of sets of sensors to targets to
\begin{equation}
\text{maximize} \sum_{l=1}^L \omega(\sigma(t_l), t_l)
\end{equation}
with the added constraint that each sensor is assigned to at most one target. 
\label{prob:general}
\end{problem}

 Problem~\ref{prob:general} is the general assignment problem which is difficult to solve for arbitrary $\omega(\cdot)$. 
%For the specific class of monotone and submodular functions, there exists a $1/2$--approximation by a greedy algorithm~\cite{fisher1978analysis}. Submodularity captures the notion of diminishing returns, i.e., the marginal gain of assigning an additional sensor diminishes as more sensors are assigned to track the same target~\cite{fisher1978analysis}. Not all measures are always submodular or monotone. 
In order to solve the assignment problem for arbitrary value functions, we consider a restricted case where each target is tracked by exactly two sensors. 

\begin{problem}[\revo{Non-overlapping} Pair Assignment] Given a set of sensors, $\mathcal{S} := \{s_0,\ldots,s_N\}$ and a set of targets, $\mathcal{T} := \{t_0,\ldots,t_L\}$, find an assignment of \revo{non-overlapping} pairs of sensors to targets:
\begin{equation}
\text{maximize} \sum_{l=1}^L \omega(\sigma_1(t_l),\sigma_2(t_l), t_l)
\end{equation}
with the added constraint each sensor is assigned to at most one target. That is, for all $i = 1,\ldots, N$ we have $|\sigma^{-1}(s_i)|\leq 1$, assuming $N \geq 2 L$. 
\label{prob:unique}
\end{problem}

\subsection{Contributions}

\revo{The main contributions of this paper are summarized as follows. 
\begin{enumerate}
    \item We derive a lower bound on the inverse of the condition number of the observability matrix which is useful when the control input for the target is not known. 
    \item We show that the condition number of the observability matrix is neither submodular nor monotone increasing. We show that the trace and the rank of the symmetric observability matrix are submodular and monotone increasing. We also show the log determinant of the symmetric observability matrix is submodular and monotone when the symmetric observability matrix is non-singular. 
    \item We present a greedy assignment algorithm that yields a $1/2$--approximation for Problem~\ref{prob:general} and $1/3$--approximation for Problem \ref{prob:unique}. 
    \item We verify the performance of our proposed greedy algorithm with various observability measures (that not necessary submodular) through extensive simulations. 
\end{enumerate}
}
%   We use this lower bound as our measure, i.e., $\omega(\cdot)$, to find the assignments in Problem \ref{prob:unique}. However, we show that this measure is not submodular (in fact, not even monotone increasing). Instead, we can use a number of other measures (e.g., the trace and the rank of the symmetric observability matrix) which we know to be submodular and monotone increasing. \rev{We also show the log determinant of the symmetric observability matrix is submodular and monotone when the symmetric observability matrix is non-singular} (Theorem~\ref{thm:submodular_observability_measure}).  
  
We start by showing how to bound the inverse of condition number (Section~\ref{sec:lowerbound}) and then present the assignment algorithms (Section~\ref{sec:algo})
%%%%%%%%%%%%%%%%%%%%%%%%%%%%%%%%%%%%%%%%%%%%%%%%%%%%%%%%%%%%%%%%%%%%%%%%%%%%%%%

\section{Bounding the Observability} \label{sec:lowerbound}
Consider a mobile target $t_l$ whose position is denoted by $p_{t_l}$. Suppose there are $N$ stationary sensors that can measure the distance\footnote{We use the square of the distance/range for mathematical convenience.} to the target. We have:
\begin{equation}
\left\{
                \begin{array}{ll}
                  \dot{p}_{t_l}=u_{t_l},\\
                  z_{s_i}=h_i(p_{t_l})=\frac{1}{2}\|p_{s_i}-p_{t_l}\|_{2}^{2}, ~i=1,...,N
                \end{array}
              \right.
\label{eqn:multi_sensor_target_system}
\end{equation}
where $p_{t_l}:=[x_{t_l},y_{t_l}]^{T}$ gives the 2D position of the target, and $u_{t_l}:=[u_{lx},u_{ly}]$ defines its control input, which is unknown to the sensor. Let $u_{t_l,\max}\triangleq\max\|u_{t_l}\|_{2}$. $z_{s_i}$ defines the range-only measurement from each sensor $s_i$ whose position is given by $p_{s_i}=[x_{s_i},y_{s_i}]^{T}$. For simplicity, we also assume that the target does not collide with any sensor, i.e., $\|p_{s_i}-p_{t_l}\|_{2} \neq 0$ and no two sensors are deployed at the same position. 

We analyze the weak local observability matrix, $O(p_{t_l},u_{t_l})$, of this multi-sensor target tracking system. We show how to lower bound the inverse of the condition number of $O(p_{t_l},u_{t_l})$, given by $C^{-1}(O(p_{t_l},u_{t_l}))$, independent of $u_{t_l}$. We also show that the lower bound, $\underline{C}^{-1}(O(p_{t_l},u_{t_l}))$, is tight.

We compute the local nonlinear observability matrix \rev{by using Lie derivatives}~\cite{williams2015observability,hermann1977nonlinear} for this system (Equation \ref{eqn:multi_sensor_target_system}) as, 
\begin{equation}
O(p_{t_l},u_{t_l})=\begin{bmatrix}
    \nabla L_{0}^{h_1} \\ 
     \nabla L_{1}^{h_1}\\ 
     \vdots\\
     \nabla L_{0}^{h_2} \\ 
     \nabla L_{1}^{h_2}\\ 
     \vdots\\
     \\
    \vdots\\
    \\
    \nabla L_{0}^{h_N} \\ 
     \nabla L_{1}^{h_N}\\ 
     \vdots
    
\end{bmatrix}=\begin{bmatrix}
    x_{t_l}-x_{s_1},y_{t_l}-y_{s_1} \\ 
     u_{lx},u_{ly}\\ 
     0,0\\
     \vdots\\
     x_{t_l}-x_{s_2},y_{t_l}-y_{s_2} \\ 
     u_{lx},u_{ly}\\ 
     0,0\\
     \vdots\\
     \\
    \vdots\\
    \\
    x_{t_l}-x_{s_N},y_{t_l}-y_{s_N} \\ 
     u_{lx},u_{ly}\\ 
     0,0\\
     \vdots
    
\end{bmatrix}.
\label{observability_multi_sensor_target}
\end{equation}
This equation can be rewritten as,
\begin{eqnarray}
O(p_{t_l},u_{t_l})&=&\begin{bmatrix}
    x_{t_l}-x_{s_1},y_{t_l}-y_{s_1} \\ 
     x_{t_l}-x_{s_2},y_{t_l}-y_{s_2} \\ 
    \vdots\\
    x_{t_l}-x_{s_N},y_{t_l}-y_{s_N} \\ 
     u_{lx},u_{ly}
\end{bmatrix}_{(N+1)\times 2}.
\label{eqn:new_multi_sensor_target}
\end{eqnarray}
The state of the target $t_l$ is \emph{weakly locally observable} if the local nonlinear observability matrix has full column rank~\cite{hermann1977nonlinear}. However, the rank test for the observability of the system does not tell the degree of the observability or how \emph{good} the observability is.  The \emph{condition number}~\cite{krener2009measures}, defined as the ratio of the largest singular value to the smallest, can be used to measure this degree of unobservability. We use the \emph{inverse of condition number} given as,
\begin{equation}
C^{-1}(O(p_{t_l},u_{t_l}))=\frac{\sigma_{\min}(O(p_{t_l},u_{t_l}))}{\sigma_{\max}(O(p_{t_l},u_{t_l}))}
\label{eqn:inverse_condition_number}.
\end{equation}
Note that, $C^{-1} \in [0,1]$. $C^{-1}=0$ means  $O(p_{t_l},u_{t_l})$ is singular and $C^{-1}=1$ means $O(p_{t_l},u_{t_l})$ is \emph{well conditioned}. A larger $C^{-1}$ means better observability (see more details in~\cite{arrichiello2013observability}).

In the local nonlinear observability matrix $O(p_{t_l},u_{t_l})$, $u_{t_l}$ is unknown and not controllable by the sensor. On the other hand, $p_{t_l}-p_{s_i}$, depends on the relative state between each sensor $s_i$ and target $t_l$ and is known to the sensor (assuming an estimate of the target's position is known). Thus, we first partition $O(p_{t_l},u_{t_l})$ into the known and unknown parts as, 
\begin{eqnarray}
O(p_{t_l},u_{t_l})=\begin{bmatrix}
    O(p_{t_l}) \\ 
     O(u_{t_l})  
\end{bmatrix},
\label{eqn:new_multi_sensor_target_partition}
\end{eqnarray}
where
\begin{equation}
O(p_{t_l}):=\begin{bmatrix}
    x_{t_l}-x_{s_1},y_{t_l}-y_{s_1} \\ 
     x_{t_l}-x_{s_2},y_{t_l}-y_{s_2} \\ 
    \vdots\\
    x_{t_l}-x_{s_N},y_{t_l}-y_{s_N}
\end{bmatrix},
\label{eqn: state_contribution_observability}
\end{equation}
and 
\begin{equation}
O(u_{t_l}):=\begin{bmatrix}
     u_{lx},u_{ly}   
\end{bmatrix},
\label{eqn: control_contribution_observability}
\end{equation}
indicate the contributions from the sensor-target relative state and target's control input, respectively. Since $O(u_{t_l})$ is unknown, we cannot compute $C^{-1}(O(p_{t_l},u_{t_l}))$ exactly. Instead, we compute its lower bound as below. 
\begin{thm}
For the multi-sensor-target system (Equation \ref{eqn:multi_sensor_target_system}) with the number of sensors, $N\geq 2$, the inverse of the condition number is lower bounded by $\dfrac{\sigma_{\min}(O(p_{t_l}))}{\sqrt{\sigma^{2}_{\max}(O(p_{t_l}))+ u_{t_l,\max}^{2}}}$.
\label{thm:lower_bound}
\end{thm}
We present the full proof for this and other results in the appendix.

We wish to improve the worst case, i.e., the lower bound of $C^{-1}(O(p_{t_l},u_{t_l}))$, by optimizing the sensor-target relative state by assigning a different subset of sensors to the target. In the following, we will show that at least two sensors are required to improve the lower bound. 
\begin{cor}
The lower bound of the observability metric in one-sensor-target system, $\underline{C}^{-1}(O_{i}(p_{t_l},u_{t_l}))$, \revo{ is identically zero. It} does not depend on the position of the sensor and therefore cannot be controlled by the sensors.
\label{cor:sensor_n_1}
\end{cor}
When the number of sensors, $N\geq 2$, we have a positive result that shows that the sensors can improve the lower bound on the inverse of the condition number of optimizing their positions, \revo{even though the contribution to the observability matrix from the target's input, $O(u_{t_l})$, is unknown and cannot be controlled.}
\begin{cor}
Suppose that the number of sensors assigned to $t_l$, $N \geq 2$. If the sensors increase $C^{-1}(O(p_{t_l}))$ and $\sigma_{\min}(O(p_{t_l})\revo{)}$ (the inverse of condition number and the smallest singular number of the relative state contribution $O(p_{t_l})$), then the lower bound of $C^{-1}(O(p_{t_l},u_{t_l}))$ also increases. 
\label{cor:sensor_n_gre2}
\end{cor} 
% We present the full proof and Remark \ref{remark:rem2} in the appendix.
\begin{rem}
The lower bound $\underline{C}^{-1}(O(p_{t_l},u_{t_l}))$ is tight when the target is known to be stationary. If $u_o\in\{0\}$, $O(p_{t_l},u_{t_l})=O(p_{t_l})$ by Equation~\ref{eqn:new_multi_sensor_target} and Equation~\ref{eqn: state_contribution_observability}. Thus, from Equation~\ref{eqn:inverse_condition_number} and Theorem~\ref{thm:lower_bound}, the lower bound  $\underline{C}^{-1}(O(p_{t_l},u_{t_l}))={C}^{-1}(O(p_{t_l},u_{t_l}))$, which implies that the lower bound is tight.
\end{rem}

Next, we use these results for assigning sensors to targets.
%%%%%%%%%%%%%%%%%%%%%%%%%%%%%%%%%%%%%%%%%%%%%%%%%%%%%%%%%%%%%%%%%%%%%%%%%%%%%%%
\section{Assignment Algorithms} \label{sec:algo}
So far, we have assumed that we know the true position, $p_{t_l}(k)$, of the target at time $k$. In practice, we only have an estimate, $\hat{p}_{t_l}(k)$, for $t_l$ along with its covariance $\Sigma_{t_l}(k)$. The estimate is obtained by fusing measurements using, for example, an Extended Kalman Filter (EKF). \revo{We use the estimated position of the target for the assignment algorithms in the simulation (Section~\ref{sec:simulation}).} 
% \pcomment{We first consider the simple case when the target is moving on some unknown trajectory, independent of the sensors. At each timestep, we compute the best pair of sensors to track the target using Algorithm~\ref{algorithm:best_pair}.}

% \begin{algorithm}
% $k\leftarrow 0$\\
% \While{true}{
%  %\KwData{this text}
%  %\KwResult{how to write algorithm with \LaTeX2e }
% $(i,j)\leftarrow$ pair that maximizes $\underline{C}^{-1}(O_{i,j}(\hat{o}(k-1),u_{o,\max}))$\\
% $\{z_i(k),z_j(k)\}\leftarrow$ obtain measurements for $i$ and $j$\\
% $\hat{o}(k), \Sigma(k) \leftarrow$ EKF update with $z_i,z_j$\\
% $k\leftarrow k + 1$
% }   
%  \caption{Best Pair Strategy}
%  \label{algorithm:best_pair}
% \end{algorithm}
% Since the sensors know the probability distribution of  target's state with mean $\hat{p}_{t_l}$, and covariance $\Sigma$, they can calculate the relative distances, $d_{io}$, by using the Mahalanobis distance \cite{mahalanobis1936generalized}, 
% \begin{equation*}
% d_{il}=\sqrt{(\hat{p}_{t_l}-p_{s_i})^{T}\Sigma^{-1}(\hat{p}_{t_l}-p_{s_i})}, ~i\in\{1,,...,N\}.
% \label{eqn:mahalanobis}
% \end{equation*}
% \emph{Perfect Pair Assignment using Bipartite Graph Matching} 
% %\label{sec:perfectmatching}

% Since pair sensor can achieve a good performance (Theorem \ref{thm:$ij_pair_extreme$}), we firstly study the assignment problem where pair sensor is assigned to a target. If each sensor pair can be assigned to a target once, and each target can be measured by only one sensor pair, 
\subsection{General assignment as Submodular Welfare Optimization}\label{sec:general_assign}
 We first study the General Assignment (Problem~\ref{prob:general}) where each target $t_l$ is tracked by a subset of sensors $\sigma(t_l)\subset \mathcal{S}, ~l\in\{1,2,...,L\}$ whose cardinality is not necessarily two. 
% We formulate this assignment problem as General Assignment (Problem \ref{prob:general}).
% $S_{t_l}\subset S, ~t_l\in\{1,2,...,L\}$ whose cardinality is not necessarily 2. We formulate the General Sensor Assignment Problem as below: 
% \begin{eqnarray*}
% &&\max\sum_{t_l=1}^{L}\omega(S_{t_l})\\
% && s.t. ~\bigcup_{{t_l}=1}^{L}S_{t_l}=S, \nonumber\\
% && \mathrm{and} ~(S_{t_1},S_{t_2},...,S_{t_L}) \mathrm{~are ~disjoint.}
% \label{eqn:maximum_submodular_problem}
% \end{eqnarray*}
% where $\omega(S_{t_l})$ denotes the lower bound of the inverse of the condition number of the observability matrix for the system where target $t_l$ is measured by a set of sensors, $S_{t_l}$. 
This is known as \emph{submodular welfare problem} in the literature~\cite{vondrak2008optimal} where the objective is to maximize $\sum_{i=1}^{n}w_i(\mathcal{S}_i)$ for independent sets $\{\mathcal{S}_i| \mathcal{S}_i\subseteq \mathcal{S}, ~i=\{1,2,...n\}\}$ by using monotone and
submodular utility functions $w_i$. A greedy algorithm~\cite{fisher1978analysis} yields a $1/2$--approximation for this problem. 
% And Vondr{\'a}k\cite{vondrak2008optimal} has given 
We first show that the lower bound of the inverse of the condition number is neither monotone nor submodular. 
\rev{This is not surprising since it has been shown that similar measures for analogous versions of the controllability matrix are also not submodular or monotone increasing~\cite{summers2016submodularity}.}

\begin{thm}
The lower bound of the inverse of condition number function $\omega(\cdot)$ is neither monotone increasing nor submodular. 
\label{Them:invcond_not_sub}
\end{thm}
\begin{proof}
We prove the claim by giving two counter-examples. 

\noindent\textbf{Case 1}: Given the sensors $s_1(0,0)$, $s_2(2\sqrt{3},-9)$, $s_3(\sqrt{3},3)$ and target $t_1(\sqrt{3},1)$ with $u_{t_{1},max}=1$ in 2-D plane, %(Figure~\ref{fig:position_cond_test_nonsubmodular}-(a))
$\omega(\{s_1,s_3\})=0.5345>\omega(\{s_1,s_2, s_3\})=0.1823$, which shows $\omega(\cdot)$ is not monotone increasing. 

\noindent\textbf{Case 2}: Given the sensors $s_1(0,0)$, $s_2(2\sqrt{3},0)$, $s_3(\sqrt{3},0.1)$, $s_4(\sqrt{3},3)$ and target $t_1(\sqrt{3},1)$ with $u_{t_{1},max}=1$ in 2-D plane, %(Figure~\ref{fig:position_cond_test_nonsubmodular}-(b))
 $\omega(\{s_1,s_2,s_3\})-\omega(\{s_1,s_2\})=0.3310-0.5345=-0.2035<\omega(\{s_1,s_2,s_4,s_3\})-\omega(\{s_1,s_2, s_4\})=0.8765-0.9258=-0.0493$, which shows $\omega(\cdot)$ is not submodular. 
\end{proof}
% \begin{figure}[htb]
% \centering{
% \subfigure[\textbf{Case 1}]{
% \includegraphics[width=0.8\columnwidth]{figs/cond_test_relative_pos_case1.eps}}
% \subfigure[\textbf{Case 2}]{
% \includegraphics[width=0.8\columnwidth]{figs/cond_test_relative_pos_case2.eps}}
% }
% \caption{Positions of sensors and target in 2D plane.\label{fig:position_cond_test_nonsubmodular}}
% \end{figure}

Therefore, we focus on other measures of observability and summarize the results in Theorem~\ref{thm:submodular_observability_measure}. 

\begin{thm}
\rev{The trace and rank of the symmetric observability matrix, $\mathbb{O}(p_{t_l},u_{t_l}):=O^{T}(p_{t_l},u_{t_l})O(p_{t_l},u_{t_l})$, and of the sensor-target relative state contribution matrix, $\mathbb{O}(p_{t_l}):=O^{T}(p_{t_l})O(p_{t_l})$, are submodular \revo{(modular)} and monotone increasing. The log determinant of the two matrices is submodular and monotone increasing if the corresponding matrix is non-singular.} \label{thm:submodular_observability_measure}
\end{thm}

\rev{The reason that we also study the measures of the symmetric observability matrix by sensor-target relative state contribution is that the control input of the target is unknown, and thus the symmetric observability matrix is not available.} 
The proof \rev{provided in the appendix} is similar to proving that the trace of the Gramian, the log determinant, and the rank of the Gramian are monotone submodular~\cite{summers2016submodularity}. \rev{Note that, if $\mathbb{O}(p_{t_l})$ is singular, its log determinant is $-\infty$. In our sensor assignment case, if a single sensor is assigned to target $t_l$, $\mathbb{O}(p_{t_l})$ is always singular, i.e., $\det(\mathbb{O}(p_{t_l}))=0$ (see the proof of \revo{Corollary~\ref{cor:sensor_n_1}}). If no sensors are assigned to a target, then the matrix is not defined. Thus, at least two sensors need to be assigned to a target to ensure non-singularity of $\mathbb{O}(p_{t_l})$.} 

\subsection{A $1/3$--approximation algorithm for Problem~\ref{prob:unique}} \label{sec:selfishmatching}
Next, we study a more specific assignment (Problem~\ref{prob:unique}) where each target $t_l$ is tracked by a pair of sensors but allow the value function to be arbitrary. 
% Note that, in Maximum Weight Perfect Bipartite Matching Problem, each sensor can be paired multiple times with other sensors. If each sensor can be matched with another sensor at most once and assigned to one specified target, a Selfish Matching Problem is constructed as below:
%  %Here, we consider the case that sensor pair is assigned to one target. Generally, pair sensors work better than single sensor in terms of of observability. However, if the control input of target is known or can be estimated, sometimes, single sensor can also have a good observability performance. 
% \begin{eqnarray*}
% &&\max_{M\subset E} \omega(M)\nonumber\\
% %&& s.t. ~\mathrm{if} ~(s_i,t_{l}) \in M, \\
% %&& ~~~~~\mathrm {then} (s_i',t_l) \notin M ~\&~ (s_i,t_l') \notin M. \nonumber\\
% && ~~~~~\mathrm{if} ~((s_i,s_j),t_{l}) \in M, \nonumber\\
% && ~~~~~\mathrm {then} ((s_i',s_j),t_{l}) \notin M ~\&~ ((s_i,s_j'),t_{l}) \notin M \nonumber\\
% && ~~~~~~\&~ ((s_i,s_j),t'_{l})  \notin M.
% \label{eqn:maximum_matching_problem}
% \end{eqnarray*}
% We first study the Unique Pair Assignment (Problem \ref{prob:unique}) where each sensor $s_i$ can be matched with another sensor $s_j$ at most once and assigned to at most one specified target $t_l$. 
We propose a greedy algorithm to solve this problem. In each round, we calculate the observability metric, $\omega(\sigma_1(t_l),\sigma_2(t_l), t_l)$ for
%all twotuples $(s_i,t_{l})$ and 
all triples $(\sigma_1(t_l),\sigma_2(t_l), t_l), ~\sigma_1(t_l),\sigma_2(t_l)\in \mathcal{S},~t_l\in\mathcal{T}$, and select the triple which has the maximum $\omega(\sigma_1(t_l),\sigma_2(t_l), t_l)$, then remove 
%$\{s_i,t_l\}$ or 
$\{\sigma_1(t_l),\sigma_2(t_l)\}$ from sensor set $\mathcal{S}$ and remove $t_l$ from target set $\mathcal{T}$, respectively. We present the greedy approach in Algorithm \ref{algorithm:unique_pair_assignment} where $\omega(\textrm{GREEDY})$ denotes total value obtained by the greedy approach. In this case, we can use the inverse of the condition number as $\omega(\cdot)$.
% and $\underline{C}^{-1}(O_{\sigma_1(l),\sigma_2(l)}(\hat{o}_{l}(k-1),u_{o_{l},\max}))$ denotes the lower bound of the inverse of the condition number of the observability matrix for system involved with sensor pair $(\sigma_1(l),\sigma_2(l))$ and target $t_l$. Sensor pair can only obtain the estimate position of the target, $\hat{o}_{l}$ by using EKF and use the maximum control input of target, $u_{o_{l},\max}$ to calculate the lower bound of the inverse of condition number, which we used for the measure of the observability in Unique Pair Assignment problem.
\begin{algorithm}
$k\leftarrow 0, ~\omega(\textrm{GREEDY})\leftarrow 0$\\
\While{true}{
Compute all possible 
$\omega(\sigma_1(t_l),\sigma_2(t_l), t_l)$.\\
Pick the triple $(\sigma_1(t_l),\sigma_2(t_l), t_l)$ with maximum $\omega(\sigma_1(t_l),\sigma_2(t_l), t_l)$ defined as $\omega_{\max}$.\\ $\omega(\textrm{GREEDY})\leftarrow \omega(\textrm{GREEDY})+\omega_{\max}$.\\
% $\mathcal{S}\backslash\{s_i,s_j\}$ and $\mathcal{T}\backslash\{t_l\}$ $\leftarrow$ 
Remove $\{s_i,s_j\}$ from the sensor set $\mathcal{S}$ and remove $t_l$ from the target set $\mathcal{T}$.\\
$k\leftarrow k + 1$
}   
 \caption{Greedy \revo{Non-overlapping} Pair Assignment}
 \label{algorithm:unique_pair_assignment}
\end{algorithm}
%Here, the observability metric $\{\underline{C}^{-1}(O_{i}(\hat{o}_{l}(k-1),u_{o_{l},\max}))\}$ represents the edge weight $\omega$. We also call our greedy selfish matching algorithm as weighted greedy selfish matching algorithm (GSM), which can be modified to unweighted version by transforming weight $\omega$ to $\omega/1$ copies with unit weight $1$. The problem can be transformed from maximizing total weight to maximizing total copies. 

% \begin{lem}
% Suppose GREEDY picks triple $(s_i,s_j,t_l)$ in current round. We will receive the reward of $(s_i,s_j,t_l)$ to at most 3 triples in OPT, all of which have reward no more than that has not been gained in a previous round. 
% \label{lem:lem_greedy_optimal}
% \end{lem}
% \begin{proof}
% Given different sensors $s_i, s_j, s_k, s_p, s_q, s_r, s_x$ and different targets $t_l, t_m, t_n$. If GREEDY picks triple $(s_i,s_j,t_l)$ with reward $c_g$, optimal can choose triple $(s_i,s_j,t_l)$ with reward less than or at most equal $c_g$, or choose $(s_i,s_k,t_l)$ and $(s_j,s_p,t_m)$ with total rewards less than or at most equal to $2c_g$ or choose $(s_i,s_p,t_m)$, $(s_j,s_q,t_n)$ and $(s_r,s_x,t_l)$ with total rewards less or at most equal to $3c_g$.
% \end{proof}
% \begin{lem}
% After the last round, all pairs in OPT have been assigned.
% \label{lem:lem_optimal_fin}
% \end{lem}
% \begin{proof}
% Based on the assumption that $N\geq 2L$ in Problem~\ref{prob:unique}, the claim is always true.
% \end{proof}
\begin{thm}
$\omega({\revo{\emph{GREEDY}}}) \geq \frac{1}{3}\omega(\revo{\emph{OPT}})$ where \revo{\emph{OPT}} is the optimal algorithm for Problem~\ref{prob:unique}. The running time for Algorithm~\ref{algorithm:unique_pair_assignment} is $O(N^2L^2)$.
\label{thm:lem_optimal_fin}
\end{thm}
\begin{proof}
\revo{We first give the proof for the approximation ratio of Algorithm~\ref{algorithm:unique_pair_assignment}.} Recall that $\omega(\textrm{GREEDY})$ and $\omega(\textrm{OPT})$ will be the sum of $\omega(\cdot)$ terms of triples consisting of one target and the two assigned sensors. As a shorthand, we will denote $\omega^g(l)$ and $\omega^*(l)$ to be the values of the triple assigned to $t_l$ by GREEDY and OPT, respectively.
%We will prove the result by showing that each $\omega(\cdot)$ term in GREEDY is greater than or equal to at most three $\omega(\cdot)$ terms in OPT. We say that we will ``charge'' a value of $\omega(\cdot)$ term in GREEDY to the (at most) three terms in OPT. Every $\omega(\cdot)$ term in OPT will be charged exactly once. Therefore, the total value charged to OPT will be less three times
%We say that the $\omega(\cdot)$ term is ``charged'' to the (at most) three terms in OPT. Furthermore, every $\omega(\cdot)$ term in OPT will be charged exactly once.

%Consider all the $\omega(\cdot)$ terms in OPT. We will replace each $\omega(\cdot)$ term in OPT by some larger or equal  $\omega(\cdot)$ term in GREEDY. Therefore, the value after replacement is greater than or equal to the value of OPT. We show that any $\omega(\cdot)$ term in GREEDY can replace at most three terms in OPT. Therefore, the cost of OPT after replacement is less than equal to thrice the cost of GREEDY proving the final result.

%\rev{{Denote $\omega^{g}(l)$ as the value of the triple includes target $t_l$, $l\in\{1,\cdots, L\}$, picked by GREEDY (in $l^{th}$ round). \textcolor{red}{ arbitrarily order all the triples in OPT, for example, in ascending order. delete? since we don't need the order} Similarly, denote $\omega^{*}(k)$ as the value of the triple includes target $t_k$, $k\in\{1,\cdots, L\}$, in OPT.} 

We show that there exists a many-to-one mapping $\mathcal{M}: [1,\cdots, L] \to [1,\cdots, L]$ such that: 
\begin{enumerate}
\item $\omega^{*}(k)\leq \omega(\mathcal{M}(k))$; and \item $|\mathcal{M}^{-1}(\revo{y})|\leq 3$ \revo{for all $y\in \mathcal{Y}$ where $\mathcal{Y} \subseteq [1,..., L]$ is the range of $\mathcal{M}$.}
\end{enumerate}
That is, each triple in OPT is mapped to some triple in GREEDY whose value is at least as high and no triple in GREEDY has more than three terms in OPT mapped to it.

We first show that if such a mapping exists, then the main result holds. Then, we prove the existence of such a mapping by constructing a specific $\mathcal{M}$. 

If such a mapping $\mathcal{M}$ exists, we prove $\omega(\textrm{GREEDY}) \geq \frac{1}{3}\omega(\textrm{OPT})$ as below.
\begin{align}
\omega(OPT) &= \sum_{k=1}^{L} \omega^{*}(k)\nonumber\\ 
&\leq \sum_{k=1}^{L} \omega^{g}(\mathcal{M}(k))\nonumber\\
&\revo{= \sum_{y\in \mathcal{Y}} \omega^{g}(y) |\mathcal{M}^{-1}(y)|} \nonumber\\
& \revo{\leq 3 \sum_{y\in \mathcal{Y}} \omega^{g}(y)} \nonumber\\
&\leq 3 \sum_{l=1}^{L} \omega^{g}(l)\nonumber\\
&= 3 \omega(\text{GREEDY}).
\label{eqn:approxiamtion_ratio}
\end{align}
The first inequality is due to $\omega^{*}(k) \leq \omega^{g}(\mathcal{M}(k))$. \revo{The second equality is because $\mathcal{M}$ maps each item in $[1,...,L]$ to the set $\mathcal{Y}$. } The second inequality is due to $|\mathcal{M}^{-1}(\cdot)|\leq 3$. \revo{The third inequality is because $\mathcal{Y}\subseteq [1,...,L]$.}

Next, we show that such a mapping always exists by constructing one. We will define $\mathcal{M}$ in the order in which the triples are picked by GREEDY. Suppose GREEDY picks the triple $(s_i,s_j,t_l)$ in some round. Then $\omega^{g}(l) = \omega(s_i,s_j,t_l)$. We will map at most three triples in OPT to this triple in GREEDY. There are three cases.

% We will charge the value $\omega(s_i,s_j,t_l)$ to at most three triples in OPT that have not been chosen in previous rounds. Furthermore, if a triple in OPT is chosen in the $l^{th}$ round, then we show that the $\omega(\cdot)$ value of the triple is less than $\omega(s_i,s_j,t_l)$.

% There are three cases. 
\begin{figure}[tbh]
\centering{
\subfigure[Case 1]{
\includegraphics[width=0.25\columnwidth]{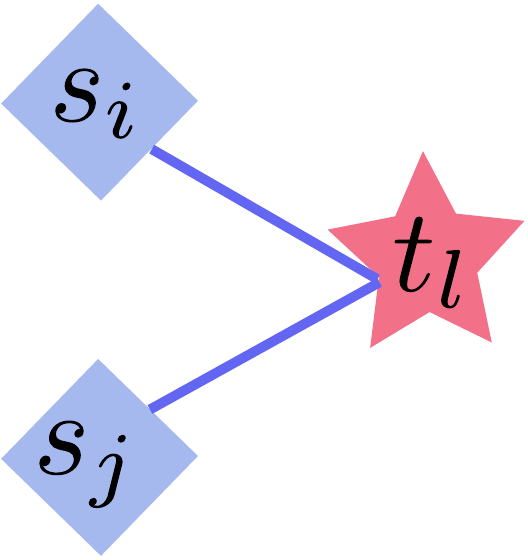}}
\subfigure[Case 2]{
\includegraphics[width=0.28\columnwidth]{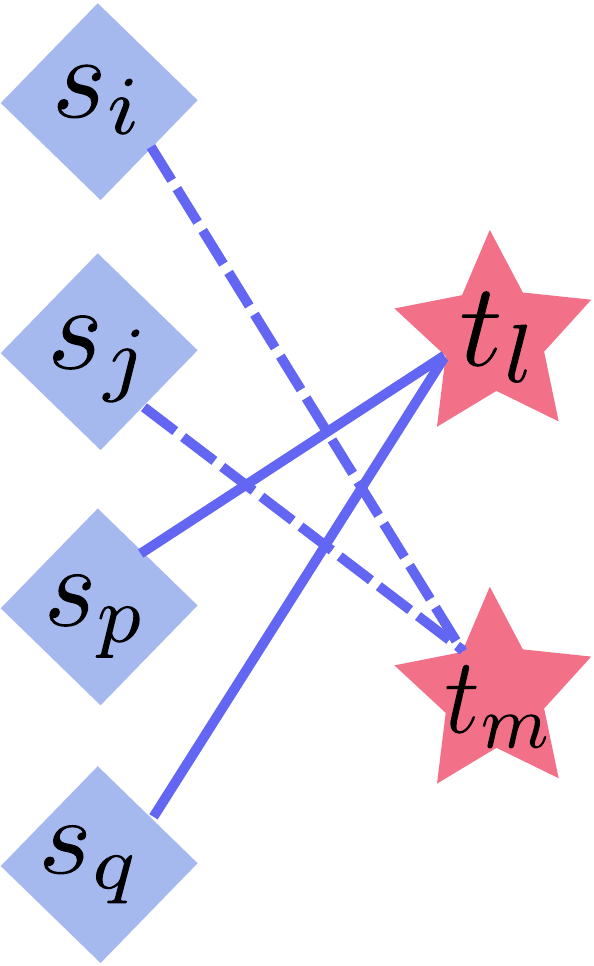}}
\subfigure[Case 3]{
\includegraphics[width=0.32\columnwidth]{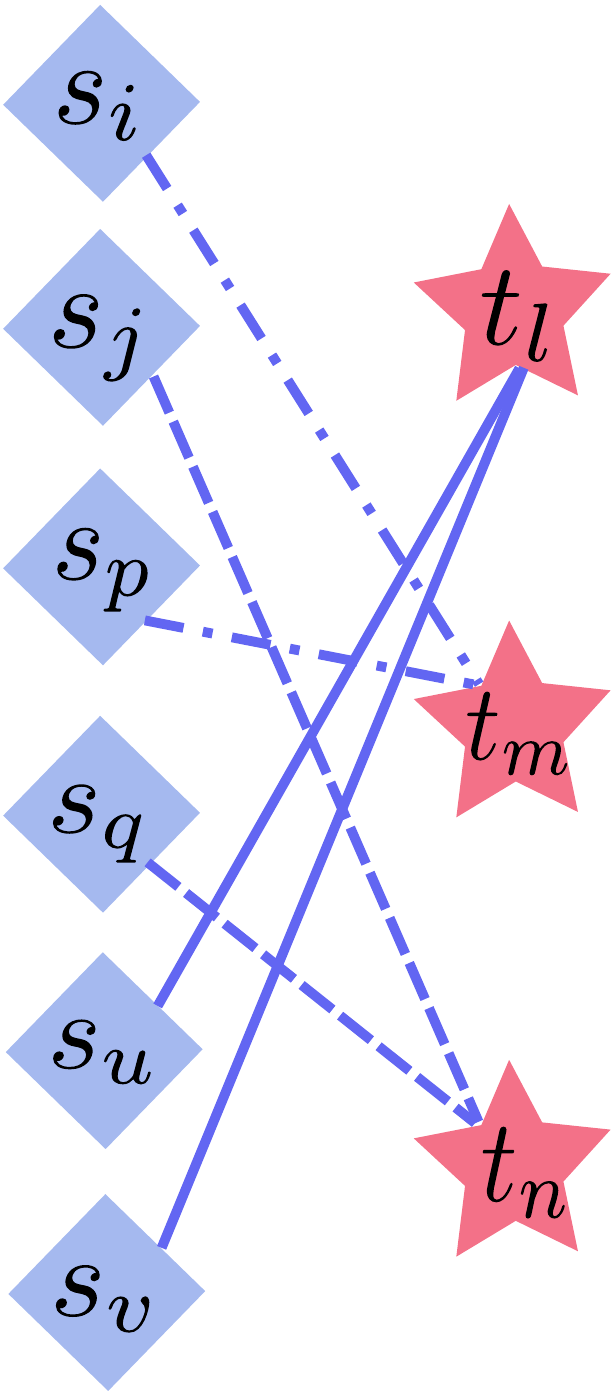}}
}
\rev{\caption{The optimal solution in three cases. \revo{In all cases, the GREEDY chooses the triple $(s_i,s_j,t_l)$. } Case 1: The OPT charges $\omega(s_i,s_j,t_l)$ to the same triple $(s_i,s_j,t_l)$ chosen by the GREEDY.  Case 2: The OPT charges $\omega(s_i,s_j,t_l)$ to at most two triples --- $(s_i,s_j,t_m)$ and  $(s_p,s_q,t_l)$. Case 3: The OPT charges  $\omega(s_i,s_j,t_l)$ to at most three triples --- $(s_i,s_p,t_m)$, $(s_j,s_q,t_n)$ and $(s_u,s_v,t_l)$. Here, $i\neq j \neq p \neq q \neq u \neq v$ and $l \neq m \neq n$. \label{fig:opt_choose}} }
\end{figure}
\begin{enumerate}
\item $(s_i,s_j,t_l)$ is also chosen by OPT (Figure~\ref{fig:opt_choose}-(a)). If $\mathcal{M}(l)$ has not been defined in previous rounds, we define $\mathcal{M}(l) = l$. Note that, here $\omega^g(l) = \omega^*(l)$ and $|\mathcal{M}^{-1}(l)| = 1$. Therefore, the two conditions for a valid mapping are satisfied.

\item  Exactly two of $(s_i,s_j,t_l)$ appear in a triple chosen by OPT (Figure~\ref{fig:opt_choose}-(b)). Consider the case where OPT chooses $(s_i,s_j,t_m)$ and $(s_p,s_q,t_l)$ where $m \neq l$. All other cases are symmetric. If $\mathcal{M}(m)$ has not been defined in a previous round, we define $\mathcal{M}(m) = l$. Note that, if $\mathcal{M}(m)$ was not defined in a previous round, then $\omega^*(m) \leq \omega^g(l)$. Otherwise, GREEDY would pick the triple $(s_i,s_j,s_m)$ in this round. Similarly, if $\mathcal{M}(l)$ is not defined in a previous round, we define $\mathcal{M}(l) = l$. By a similar argument, if $\mathcal{M}(l)$ was not defined in a previous round, $\omega^*(l) \leq \omega^g(l)$. Furthermore, $|\mathcal{M}^{-1}(l)|=2$. Therefore, the two conditions for a valid mapping are satisfied.

\item No two of $(s_i,s_j,t_l)$ appear in the same triple chosen by OPT. Suppose instead they appear in three distinct triples, $(s_u,s_v,t_l)$, $(s_i,s_p,t_m)$, and $(s_j,s_q,t_n)$ as shown in Figure~\ref{fig:opt_choose}-(c). We can use a similar argument as in the previous case. If any of $l,m$ and $n$ were not mapped in a previous round, then we will map them to $l$. Furthermore, since they were not mapped in a previous round, their value in OPT cannot be greater than $\omega^g(l)$. Finally, $|\mathcal{M}^{-1}(l)|\leq 3$. Therefore, the two conditions for a valid mapping are satisfied.
\end{enumerate}
Therefore, given such a mapping $\mathcal{M}$, it follows that $\omega(\textrm{GREEDY}) \geq \frac{1}{3}\omega(\textrm{OPT})$. 

\revo{We next prove the running time for Algorithm~\ref{algorithm:unique_pair_assignment}. The ``while" loop runs for $L$ rounds since all the targets must be tracked eventually. Inside the ``while" loop, we compute all possible triples and find the best one. This requires $O(N^2L)$ time. Overall, Algorithm~\ref{algorithm:unique_pair_assignment} runs in $O(N^2L^2)$ time.}
\end{proof}
\revo{

It is possible to generalize this result to the case where exactly $n$ sensors are to be assigned to a target with $n\geq 2$. We can obtain a generalized bound, $\omega(\text{GREEDY}) \geq \frac{1}{n+1}\omega(\text{OPT})$ by using a proof similar to that of Theorem~\ref{thm:lem_optimal_fin}. 
}

\section{Simulations}\label{sec:simulation}
We illustrate the performance of the assignment strategies for sensor selection using observability measure as the performance criterion. 
%\revone{We first consider the general assignment with the submodular monotone observability measure, i.e., the trace of the symmetric observability matrix and simulate the greedy general assignment approach~\cite{fisher1978analysis}. Then we focus on the unique pair assignment with arbitrary observability measure, i.e., the inverse of the condition number of the observability matrix and simulate and evaluate the greedy unique pair assignment (Algorithm~\ref{algorithm:unique_pair_assignment}).} 
A video showing the algorithm in action is submitted as supplementary material.

\subsection{Greedy General Assignment}
\begin{figure*}[tbh]
\centering{
\subfigure[$k=1$ (initial time)]{
\includegraphics[width=0.65\columnwidth]{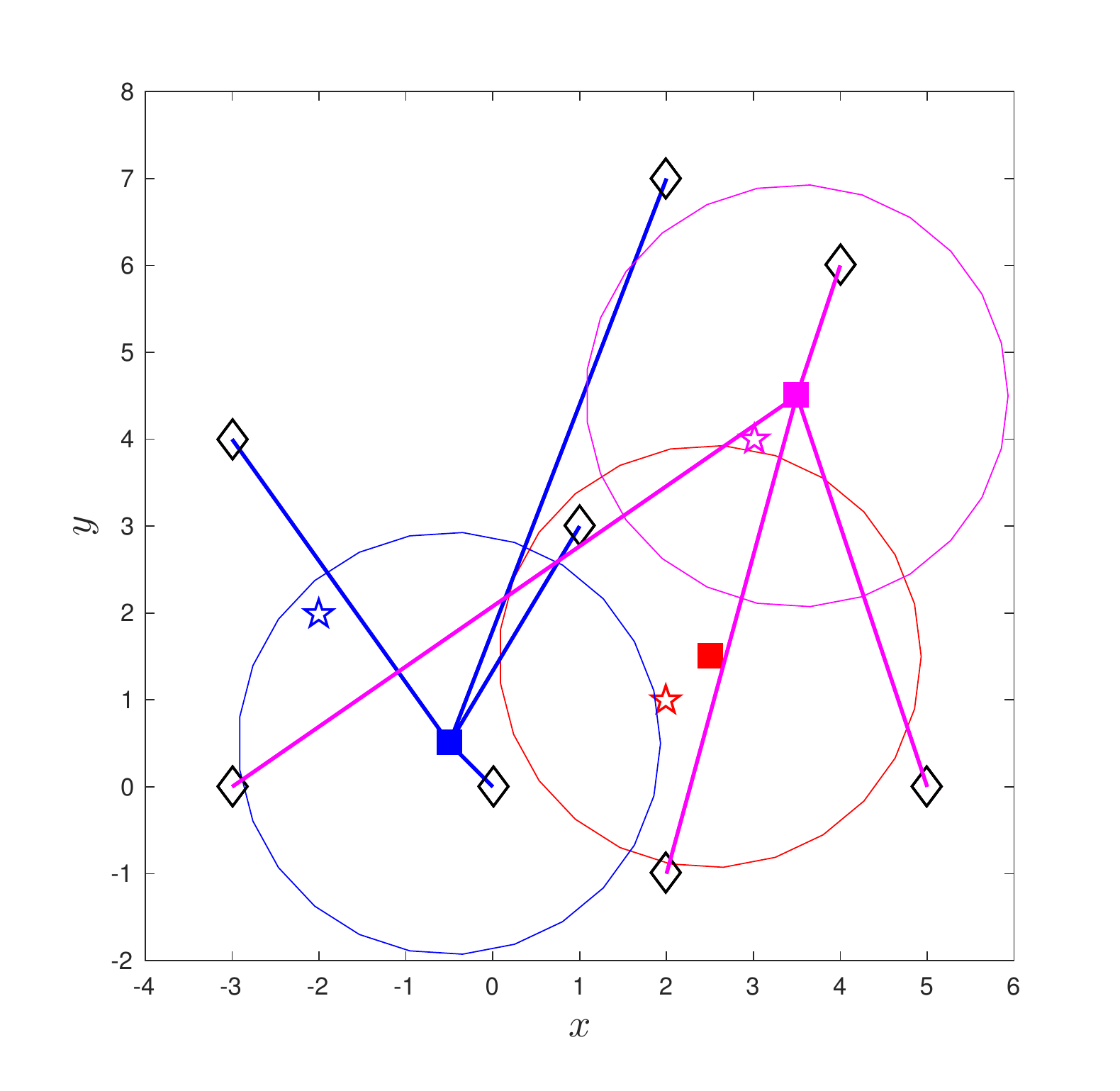}}
\subfigure[$k=60$]{
\includegraphics[width=0.65\columnwidth]{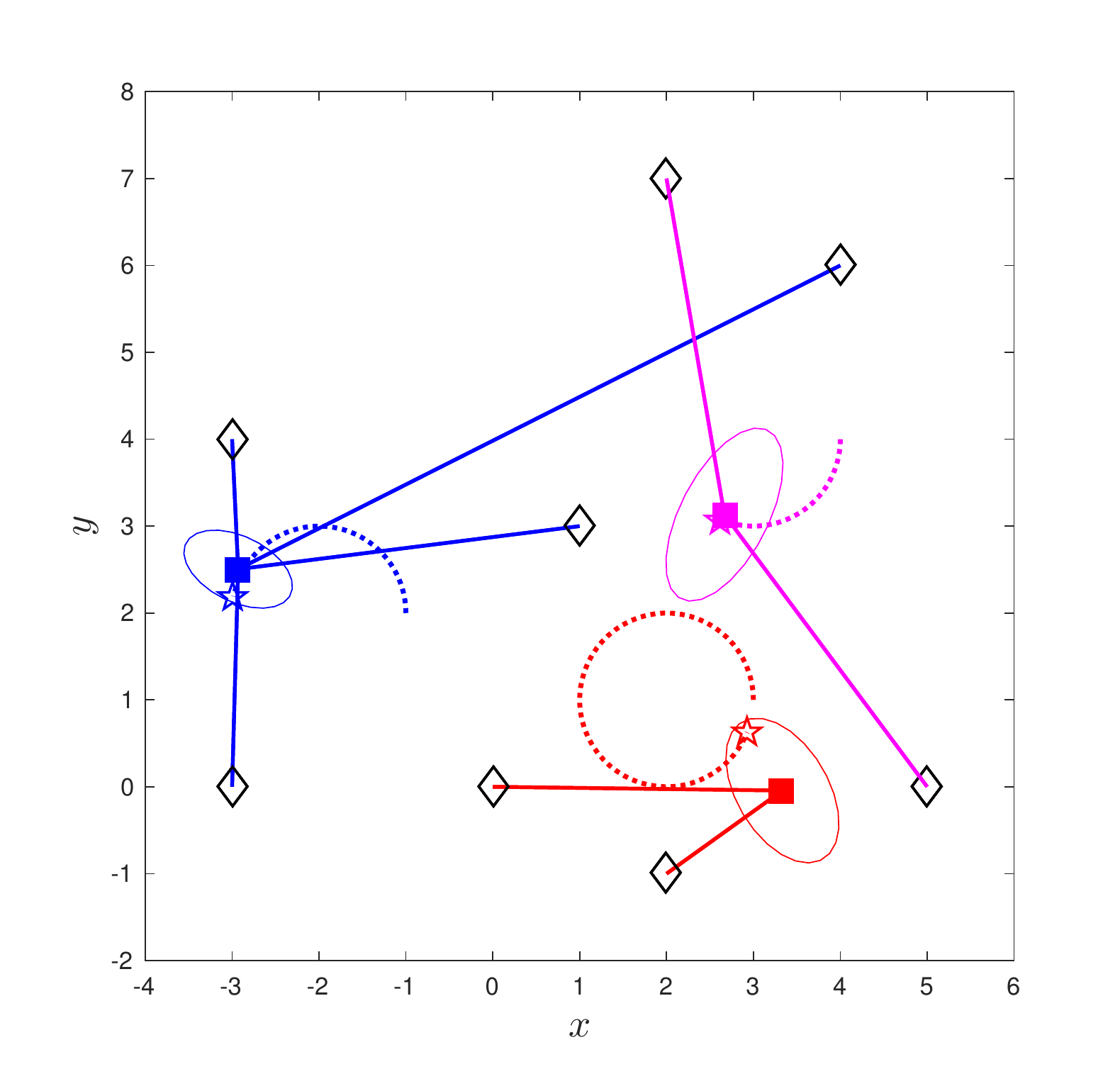}}
\subfigure[$k=100$]{
\includegraphics[width=0.65\columnwidth]{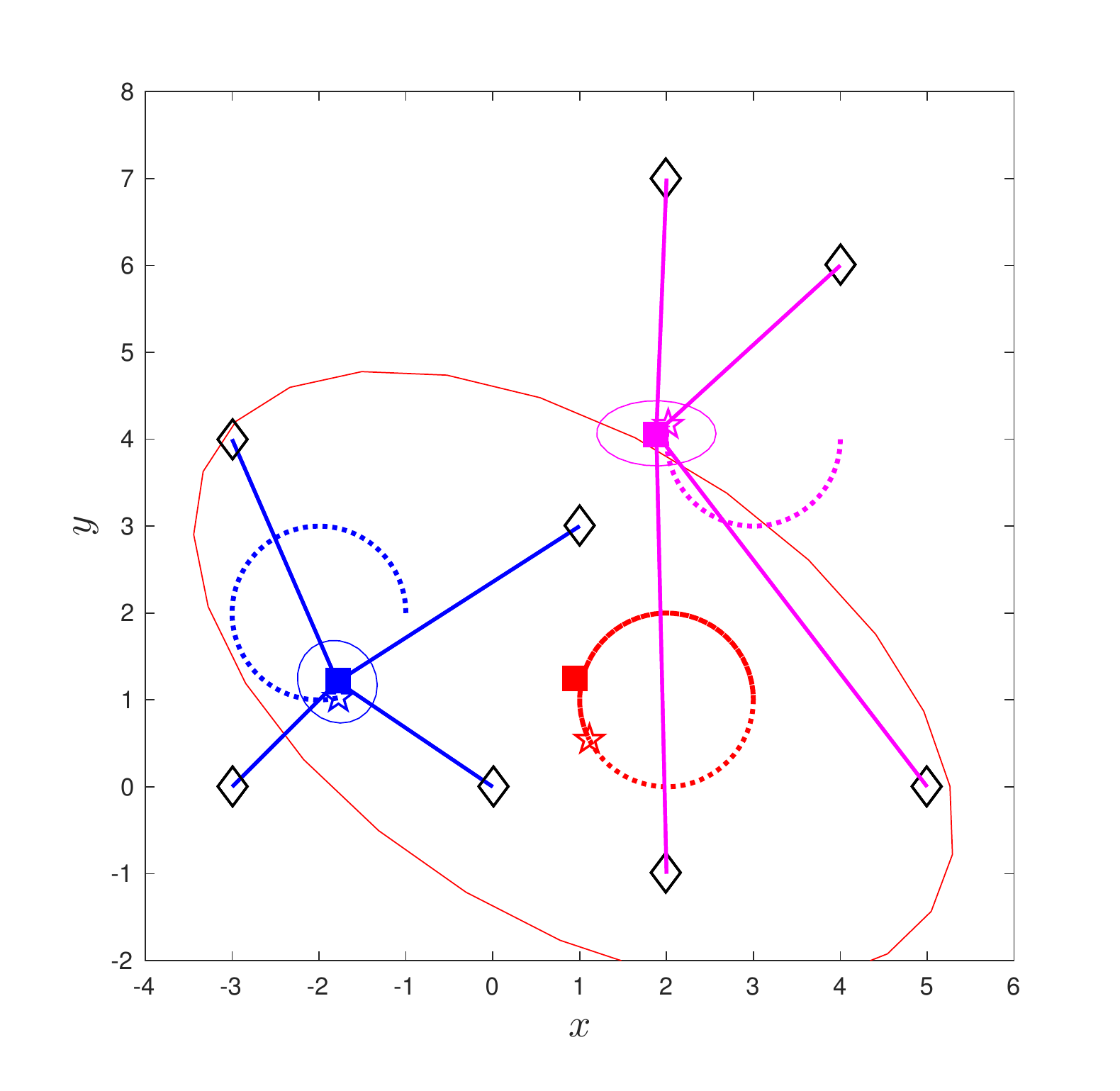}}
}
\caption{Greedy General Assignment~\cite{fisher1978analysis} in action for tracking three targets with circular motion by using $\text{trace}(\mathbb{O}(p_{t_l}))$. The three colors red, blue and magenta specify three targets, respectively. The pentagram, filled square, solid ellipse (sometimes, it looks like \revo{a} solid circle) and dotted circle indicates the true position, estimate mean position, variance and trajectory for the target, respectively. The black diamond indicates the sensor. The solid line joining the target and sensor indicates that the sensor is assigned to the target.\label{fig:gen_assign_tracking}}
\end{figure*}

% \begin{figure*}[tbh]
% \centering{
% \subfigure[$k=1$ (initial time)]{
% \includegraphics[width=0.65\columnwidth]{figs/gen_assign_k_1.eps}}
% % \subfigure[$k=20$]{
% % \includegraphics[width=0.65\columnwidth]{figs/gen_assign_k_20.eps}}
% % \subfigure[$k=40$]{
% % \includegraphics[width=0.65\columnwidth]{figs/gen_assign_k_40.eps}}
% \subfigure[$k=60$]{
% \includegraphics[width=0.65\columnwidth]{figs/gen_assign_k_60.eps}}
% % \subfigure[$k=80$]{
% % \includegraphics[width=0.65\columnwidth]{figs/gen_assign_k_80.eps}}
% \subfigure[$k=100$]{
% \includegraphics[width=0.65\columnwidth]{figs/gen_assign_k_100.eps}}
% }
% \caption{Greedy General Assignment~\cite{fisher1978analysis} in action for tracking three targets with circular motion by using $\log(det(\mathbb{O}(p_{t_l},u_{t_l})))$. The three colors red, blue and magenta specify three targets, respectively. The pentagram, filled square, solid ellipse (sometimes, it looks like solid circle) and dotted circle indicates the true position, estimate mean position, variance and trajectory for the target, respectively. The black diamond indicates the sensor. The solid line joining the target and sensor indicates that the sensor is assigned to the target.\label{fig:gen_assign_tracking}}
% \end{figure*}

We first simulate the greedy assignment~\cite{fisher1978analysis} for Problem~\ref{prob:general} by using \rev{the trace of the \emph{symmetric observability matrix by sensor-target relative state contribution}, $\text{trace}(\mathbb{O}(p_{t_l}))$}
%$\footnote{We calculate \rev{$trace(\mathbb{O}(p_{t_l},u_{t_l}))$} from the known part $O(p_{t_l})$ (Equation~\ref{eqn: state_contribution_observability}) only.} 
as observability measure which is monotone and submodular \revo{(modular) as shown in Theorem~\ref{thm:submodular_observability_measure}. Notably, for maximizing a modular function, the greedy algorithm yields an optimal solution.} We start with the greedy assignment with 8 stationary sensors and 3 targets moving in a circle and $u_{t_l,\max}=1,~ l\in\{1,2,3\}$ within 100 time steps in a $10\times 10$ environment. At each timestep, we use the greedy approach~\cite{fisher1978analysis} to assign a set of sensors to each target, and use the measurements of the set of sensors to update the estimate of the target by EKF as shown in Figure~\ref{fig:gen_assign_tracking}.

% \begin{figure}[htb]
% \centering{
% \subfigure[]{
% \includegraphics[width=0.47\columnwidth]{figs/gen_assign_tar_mean_error.eps}}
% \subfigure[]{
% \includegraphics[width=0.47\columnwidth]{figs/gen_assign_tar_tra_cova}}
% }
% \caption{(a) The mean error for each target by greedy approach~\cite{fisher1978analysis} in General Assignment within 100 time steps. (b) The trace of variance for each target by greedy approach~\cite{fisher1978analysis} in General Assignment within 100 time steps.\label{fig:gen_err_trace}}
% \end{figure}
\rev{}
In order to further evaluate the greedy approach for general assignment, we set the number of targets as $L=5$ and number of sensors, $N$, from 20 to 50. For each $N\in\{20,21,...,50\}$, the positions of sensors and targets are randomly generated within $[0,100]\times [0,100]\in\mathbb{R}^{2}$ for 30 trials. 
% We know finding the optimal solution of General Assignment problem is NP-Complete. 
We compare the number of sensors assigned to one specific target, i.e., $t_l, l\in\{1,2,...L\}$ with the $N/L$ as shown in Figure~\ref{fig:trial_submodular_obs.eps}. It shows that the sensors are assigned to each target almost evenly.  

\begin{figure}[htb]
\centering{
%\subfigure[]{\includegraphics[width=0.8\columnwidth]{figs/multi_robot_target.eps}}
\includegraphics[width=0.9\columnwidth]{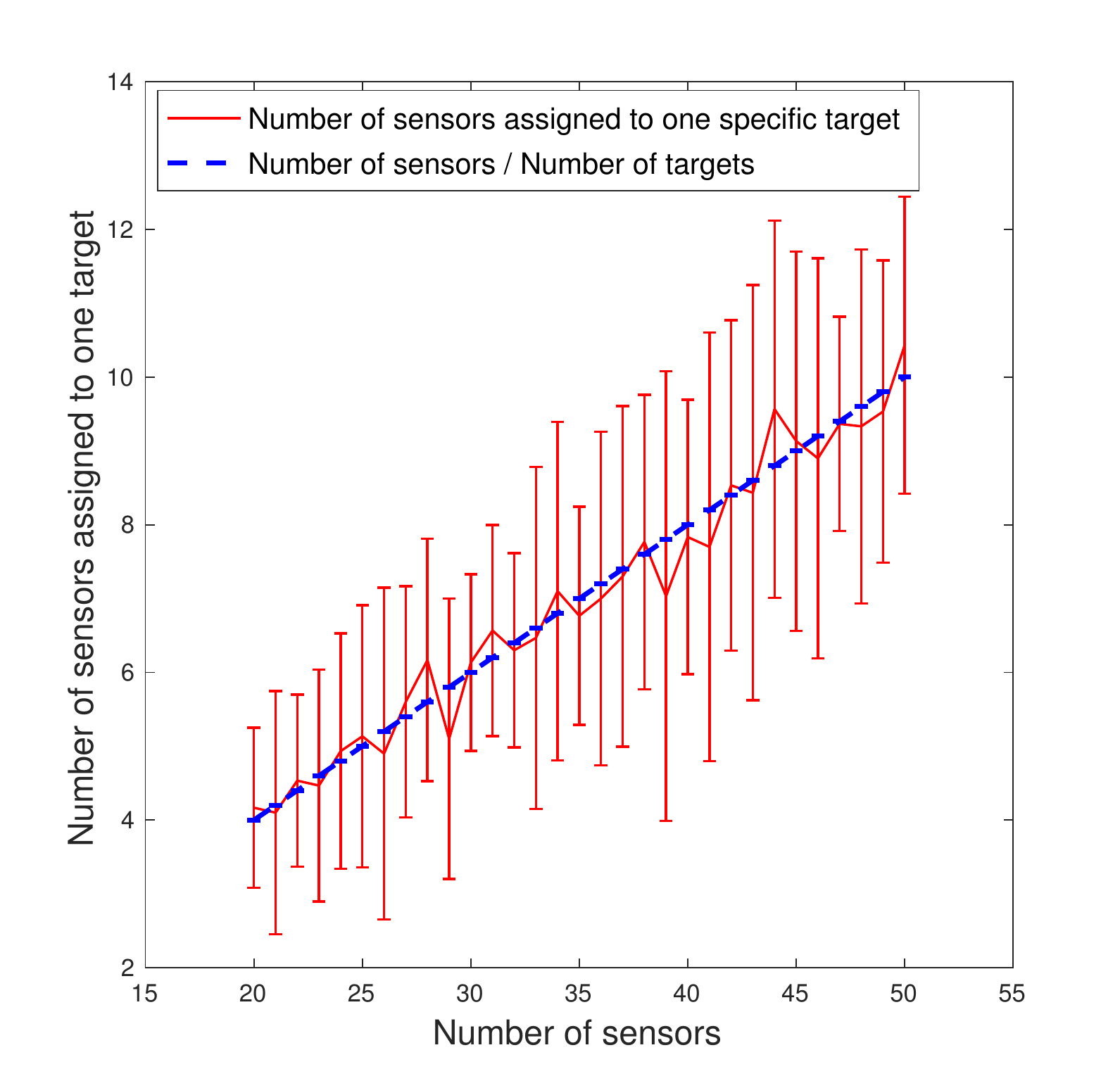}
}
\caption{Number of sensors assigned to each target.\label{fig:trial_submodular_obs.eps}} 
\end{figure}

\subsection{Greedy \revo{Non-overlapping} Pair Assignment}
We then simulate the greedy \revo{non-overlapping} pair assignment (Algorithm~\ref{algorithm:unique_pair_assignment}) for solving Problem~\ref{prob:unique}. With the same environment setting, we start with simulating Algorithm~\ref{algorithm:unique_pair_assignment} with $\text{trace}(\mathbb{O}(p_{t_l}))$  as the observability measure (Figure~\ref{fig:uni_assign_tracking_trace}). \rev{We know that the observability measure for greedy \revo{non-overlapping} pair assignment does not have to be monotone submodular. Therefore, we use the log determinant of the \emph{symmetric observability matrix by sensor-target relative state contribution},  $\log\det(\mathbb{O}(p_{t_l}))$,  and the lower bound on the inverse of the condition number of the observability matrix, $\underline{C}^{-1}(O(p_{t_l},u_{t_l}))$, for the assignment. The assignments for a specific scenario are shown in Figures~\ref{fig:uni_assign_tracking_logdet} and \ref{fig:uni_assign_tracking_invcond}. Note that, even though a pair of sensors is assigned to target $t_l$, $\mathbb{O}(p_{t_l})$ is not always non-singular.}

\begin{figure*}[thb]
\centering{
\subfigure[$k=1$ (initial time)]{
\includegraphics[width=0.65\columnwidth]{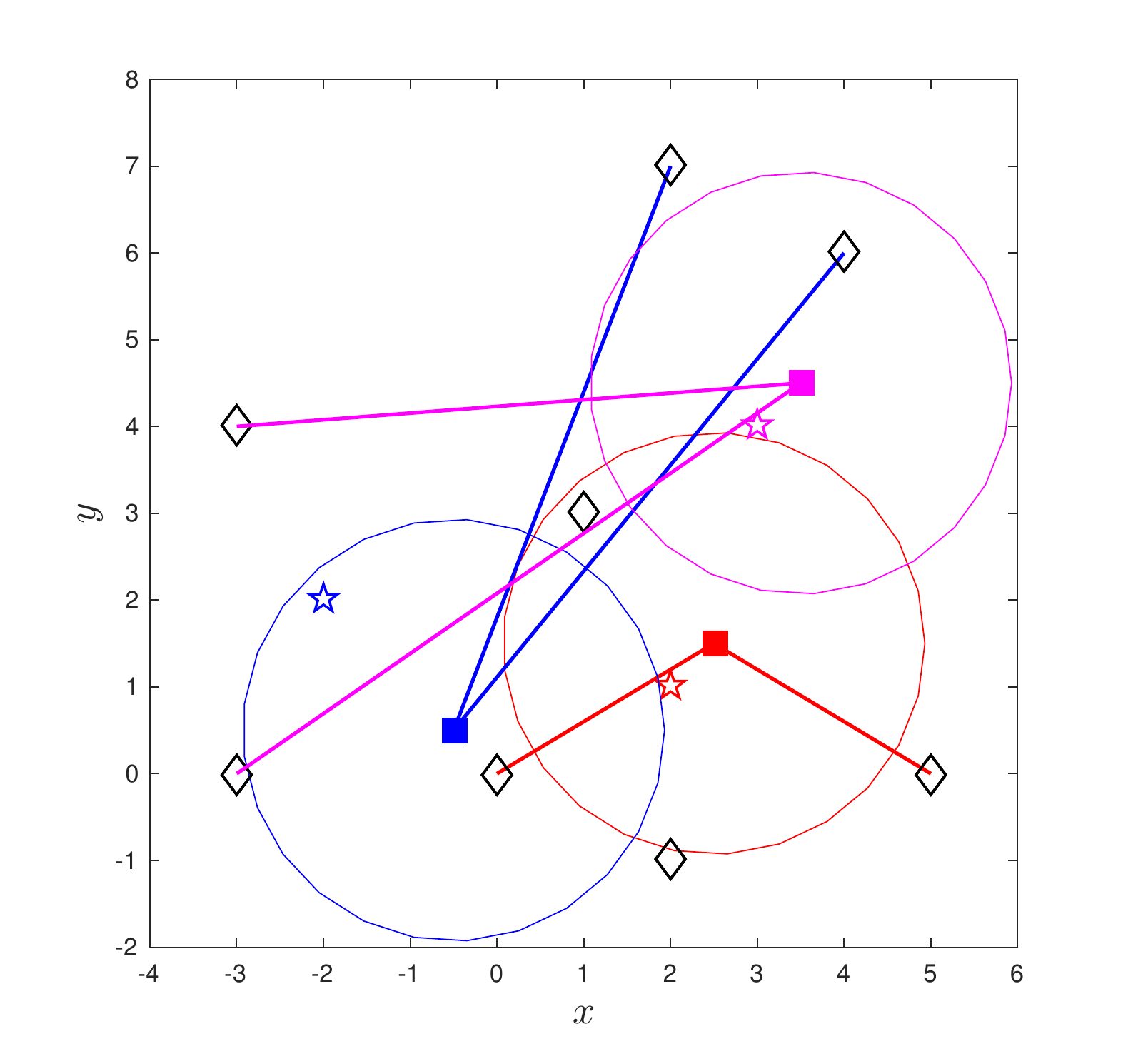}}
\subfigure[$k=60$]{
\includegraphics[width=0.65\columnwidth]{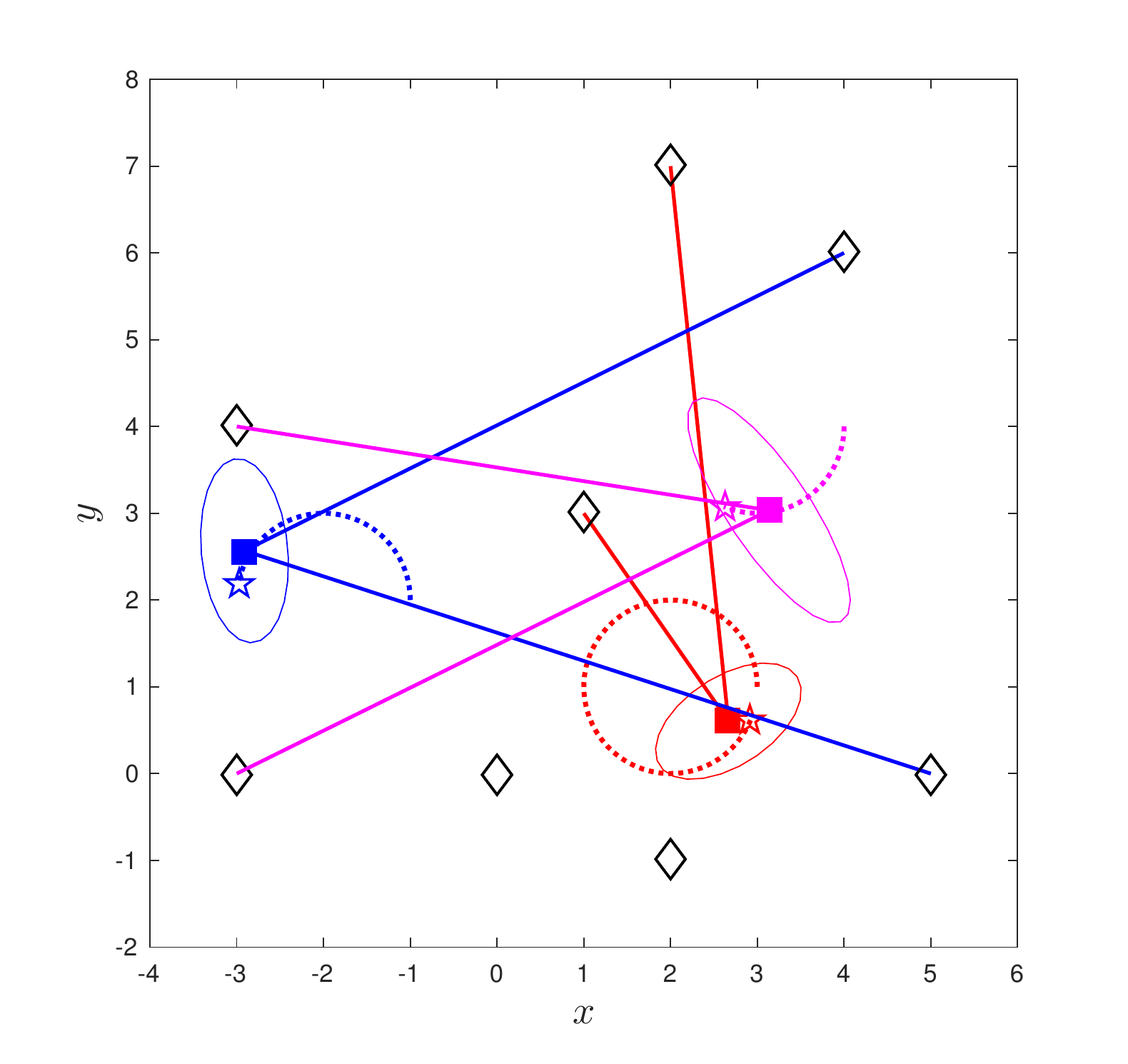}}
\subfigure[$k=100$]{
\includegraphics[width=0.65\columnwidth]{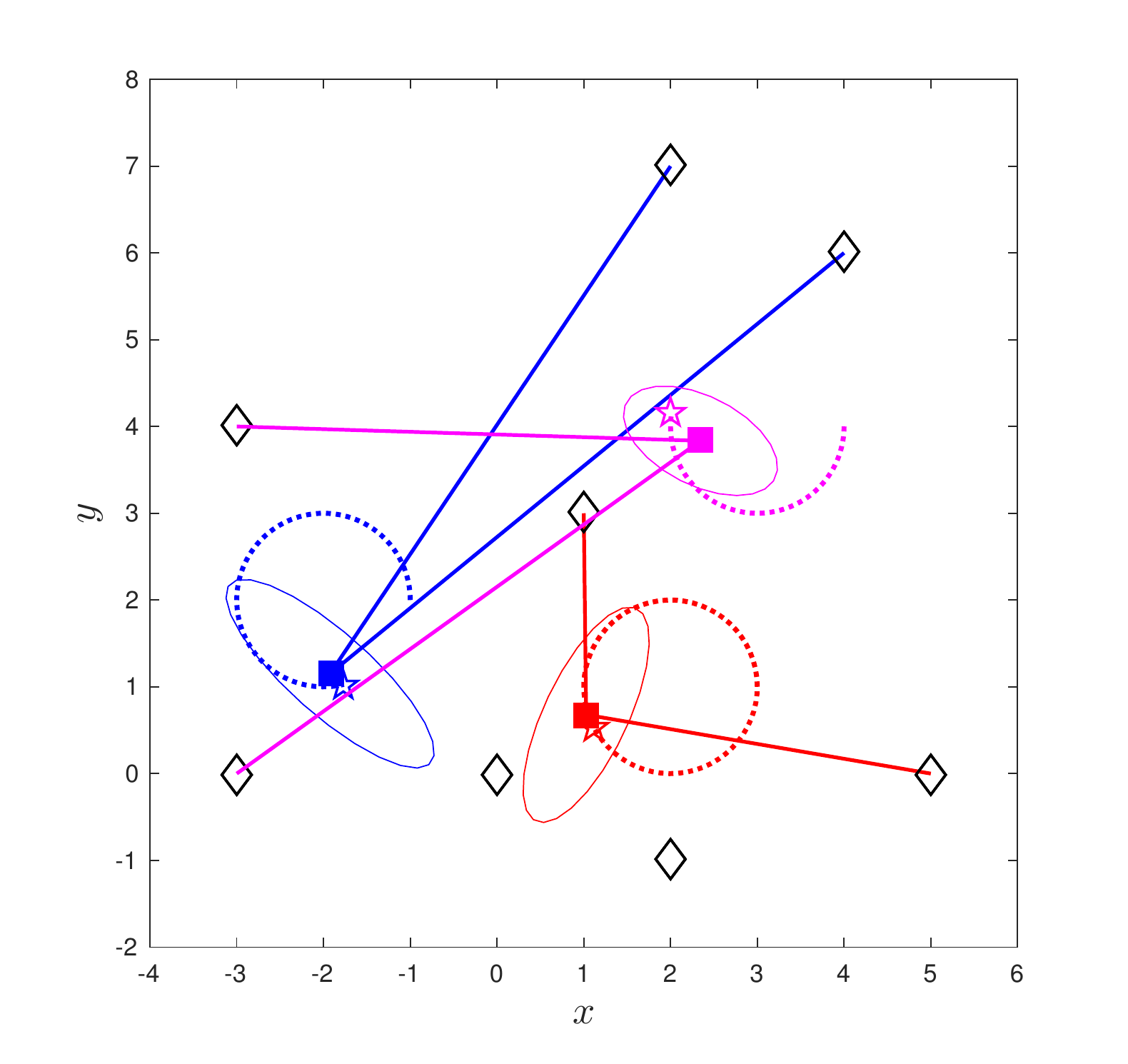}}
}
\caption{Greedy \revo{Non-overlapping} Pair Assignment (Algorithm~\ref{algorithm:unique_pair_assignment}) in action for tracking three targets with circular motion by using $\text{trace}(\mathbb{O}(p_{t_l}))$. The three colors red, blue and magenta specify three targets, respectively. The pentagram, filled square, solid ellipse (sometimes, it looks like \revo{a} solid circle) and dotted circle indicates the true position, estimate mean position, variance and trajectory for the target, respectively. The black diamond indicates the sensor. The solid line joining the target and sensor indicates that the sensor is assigned to the target.\label{fig:uni_assign_tracking_trace}}
\end{figure*}

\begin{figure*}[thb]
\centering{
\subfigure[$k=1$ (initial time)]{
\includegraphics[width=0.65\columnwidth]{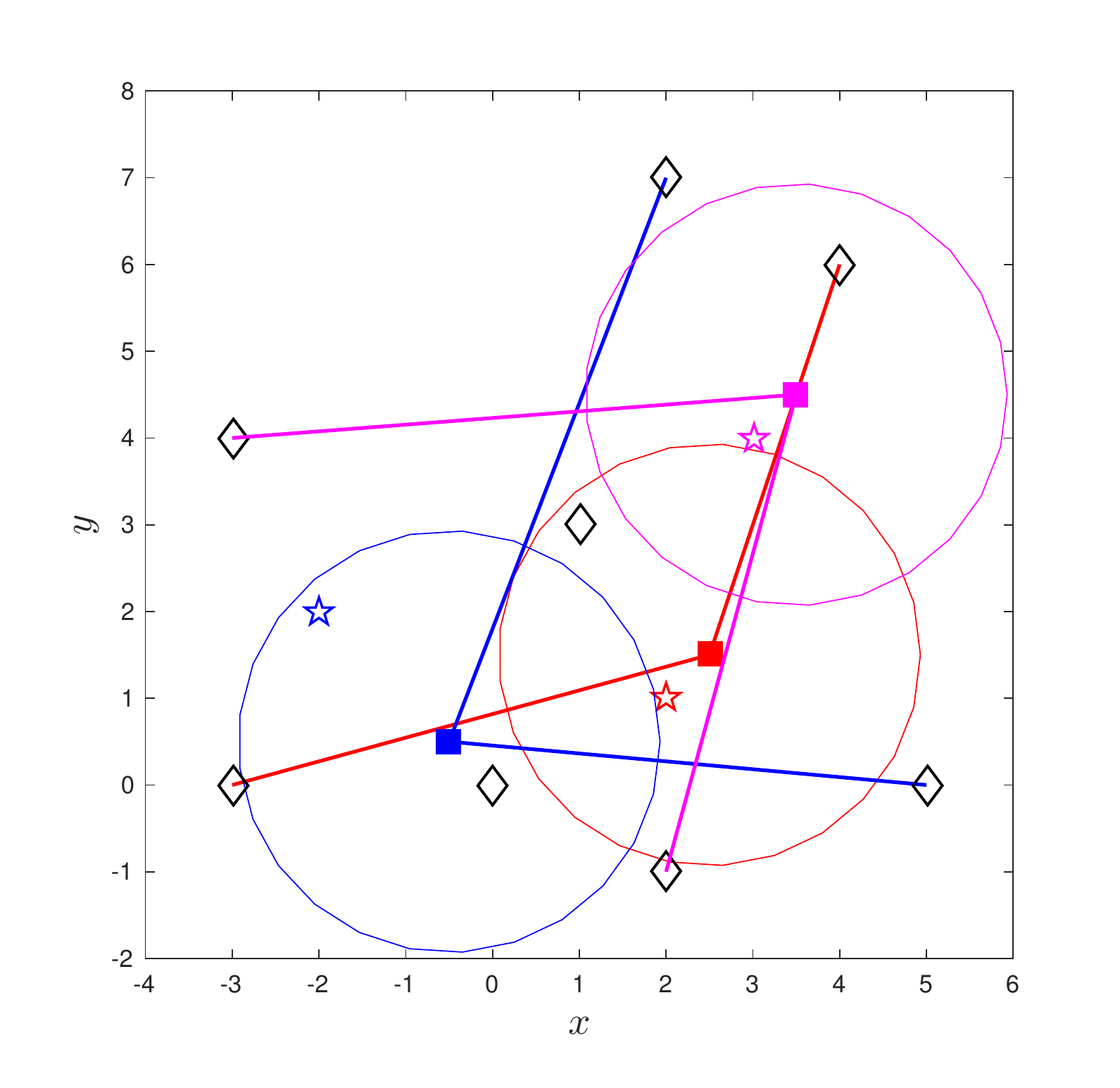}}
% \subfigure[$k=20$]{
% \includegraphics[width=0.65\columnwidth]{figs/uni_assign_log_deter_k_20.eps}}
% \subfigure[$k=40$]{
% \includegraphics[width=0.65\columnwidth]{figs/uni_assign_log_deter_k_40.eps}}
\subfigure[$k=60$]{
\includegraphics[width=0.65\columnwidth]{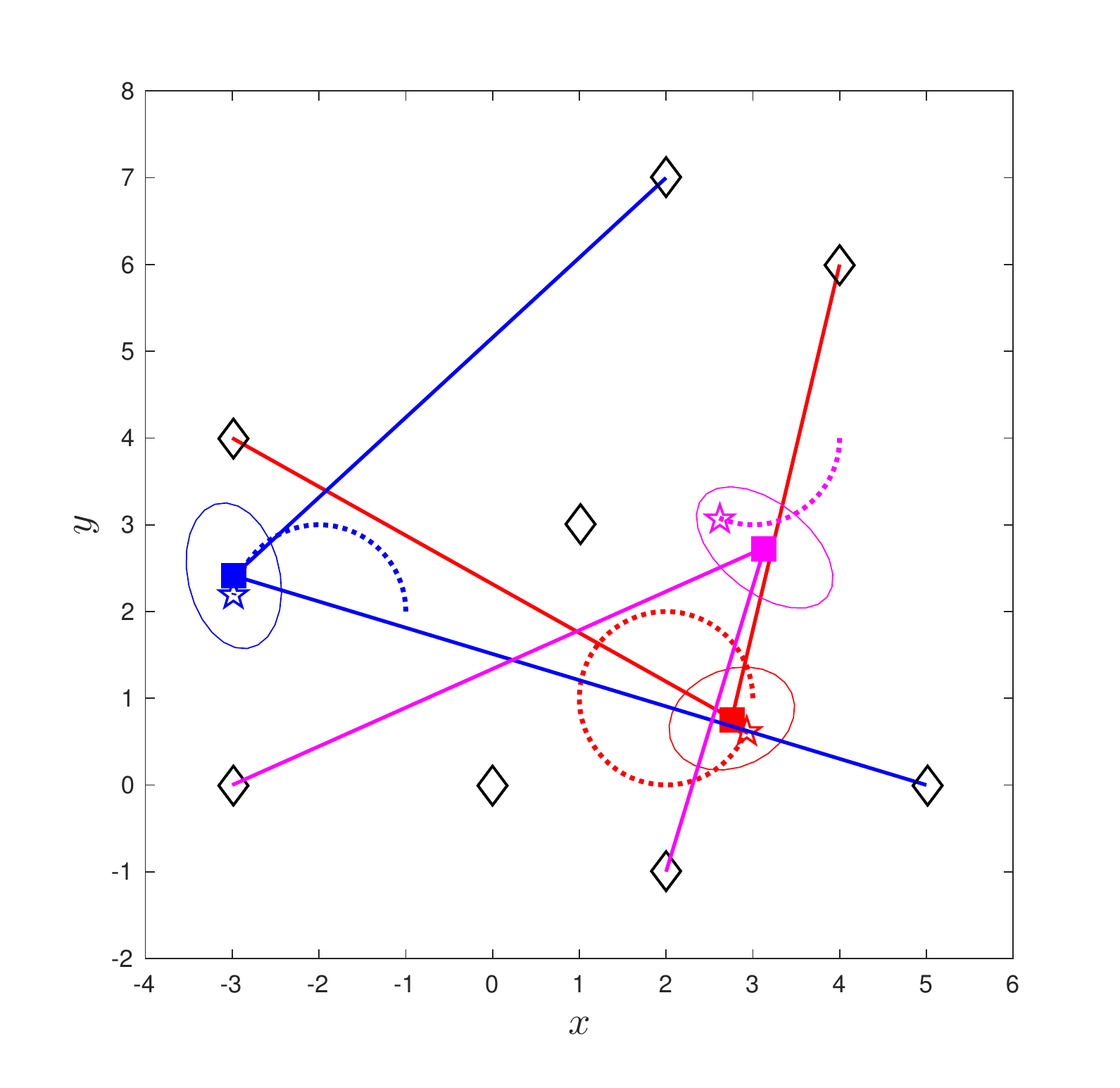}}
% \subfigure[$k=80$]{
% \includegraphics[width=0.65\columnwidth]{figs/uni_assign_log_deter_k_80.eps}}
\subfigure[$k=100$]{
\includegraphics[width=0.65\columnwidth]{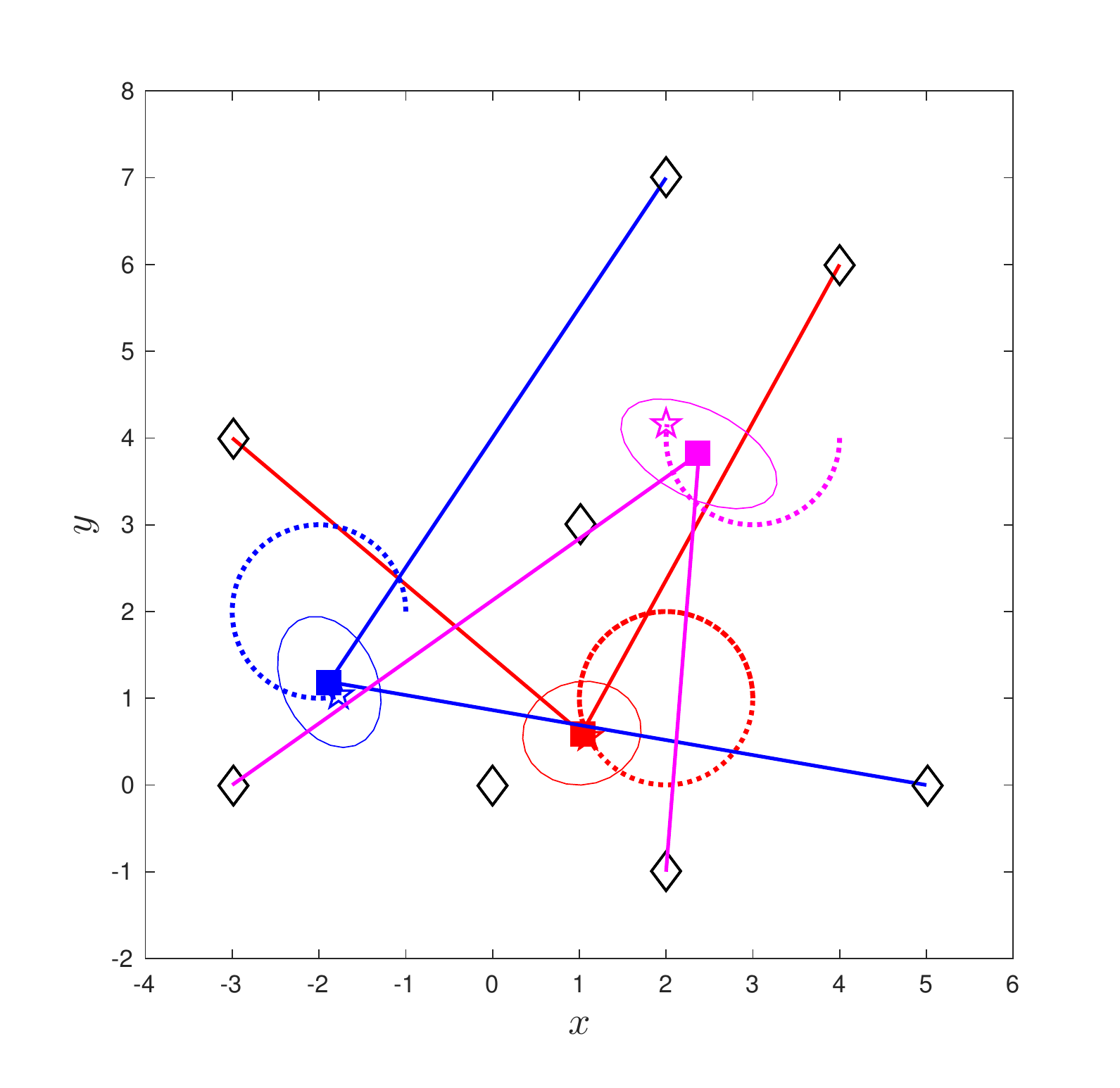}}
}
\caption{Greedy \revo{Non-overlapping} Pair Assignment (Algorithm~\ref{algorithm:unique_pair_assignment}) in action for tracking three targets with circular motion by using \rev{$\log\det(\mathbb{O}(p_{t_l}))$}. The three colors red, blue and magenta specify three targets, respectively. The pentagram, filled square, solid ellipse (sometimes, it looks like \revo{a} solid circle) and dotted circle indicates the true position, estimate mean position, variance and trajectory for the target, respectively. The black diamond indicates the sensor. The solid line joining the target and sensor indicates that the sensor is assigned to the target.\label{fig:uni_assign_tracking_logdet}}
\end{figure*}

For each target $t_l$, the estimate generated by EKF includes its mean position $\hat{p}_{t_l}$ and variance $\Sigma_{t_l}$, $l\in\{1,2,3\}$. To evaluate the tracking performance, denote the mean error as 
\begin{equation}
    \mathrm{err}_{t_l} = \|\hat{p}_{t_l}-p_{t_l}\|_2.
    \label{eqn:mean_error}
\end{equation}
 and the trace of covariance as $\mathrm{tr}(\Sigma_{t_l})$. Figure~\ref{fig:comparsion_gen_uni_log_mean_trace} shows $\mathrm{err}_{t_l}$ and $\mathrm{tr}(\Sigma_{t_l})$ for each target $t_l$ with the greedy algorithms to the two assignment problem and with $\text{trace}(\mathbb{O}(p_{t_l})\revo{)}$ as the measure. \rev{We can observe that the greedy general assignment performs better for some target (e.g., $t_2$), but performs worse for target $t_1$ as compared to the \revo{non-overlapping} pair assignment (Figure~\ref{fig:comparsion_gen_uni_log_mean_trace}-(b))}. \rev{Both algorithms maximize the sum of the observability measures for the three targets. In the \revo{non-overlapping} pair assignment, a pair of sensors is always assigned to each target whereas in the general assignment, there are situations where no sensors are assigned to a target. This can be due to the fact that it is profitable to assign more than two sensors to some target in order to maximize the sum. We can observe the happening in Figure~\ref{fig:gen_assign_tracking} where there exist some time steps (e.g., $k=100$) when less than two sensors are assigned to a target (red), which leads to a bad tracking performance for individual targets in the general assignment case (\revo{Corollary~\ref{cor:sensor_n_1}}). However, in all cases, the greedy general assignment~\cite{fisher1978analysis} ensures that the sum of individual measures is within a factor of 2 of the optimal sum.
 
Figure~\ref{fig:comparsion_uni_logcondtrace_mean_trace} shows the tracking performance with the inverse condition number and log det is comparable. The performance with trace is worse as compared to the other two.  However, we can only use the trace for the greedy general assignment, since it is submodular and monotone. This exactly reflects the importance of our greedy \revo{non-overlapping} pair assignment that better observability measures, i.e., log det and inverse condition number, which are not necessarily submodular and monotone are also allowed.}

\begin{figure*}[htb]
\centering{
\subfigure[$k=1$ (initial time)]{
\includegraphics[width=0.65\columnwidth]{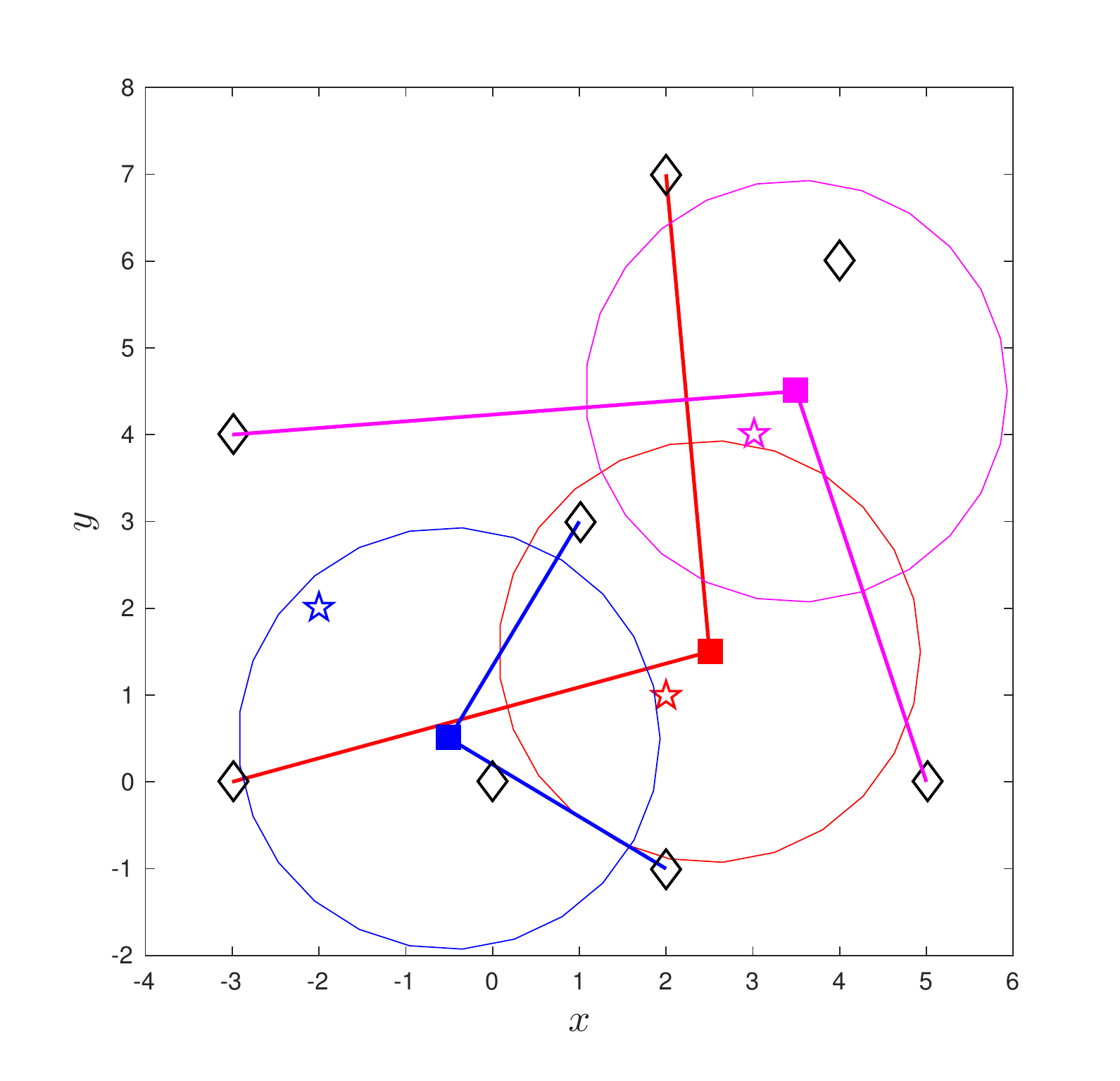}}
% \subfigure[$k=20$]{
% \includegraphics[width=0.65\columnwidth]{figs/uni_assign_k_20.eps}}
% \subfigure[$k=40$]{
% \includegraphics[width=0.65\columnwidth]{figs/uni_assign_k_40.eps}}
\subfigure[$k=60$]{
\includegraphics[width=0.65\columnwidth]{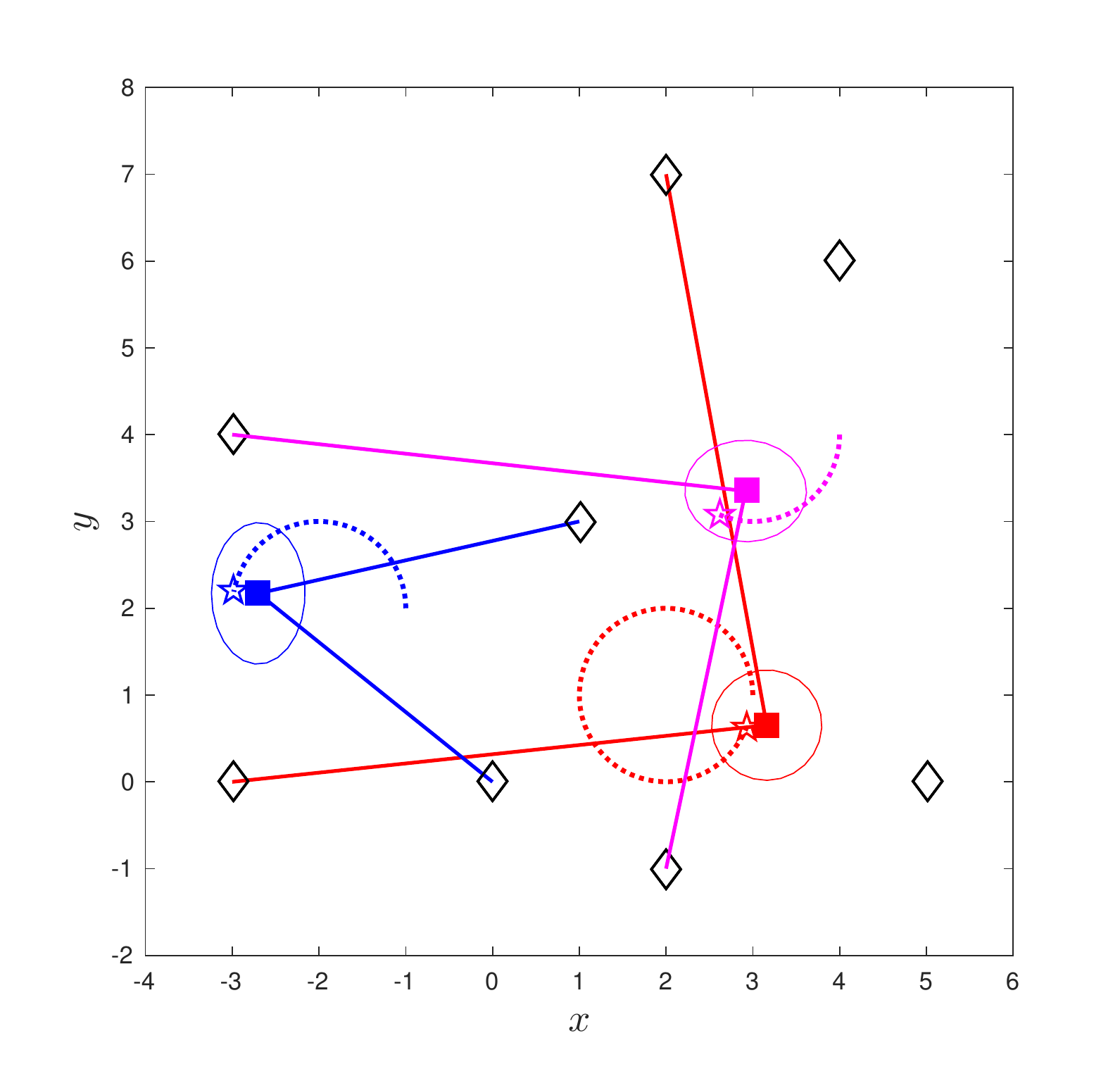}}
% \subfigure[$k=80$]{
% \includegraphics[width=0.65\columnwidth]{figs/uni_assign_k_80.eps}}
\subfigure[$k=100$]{
\includegraphics[width=0.65\columnwidth]{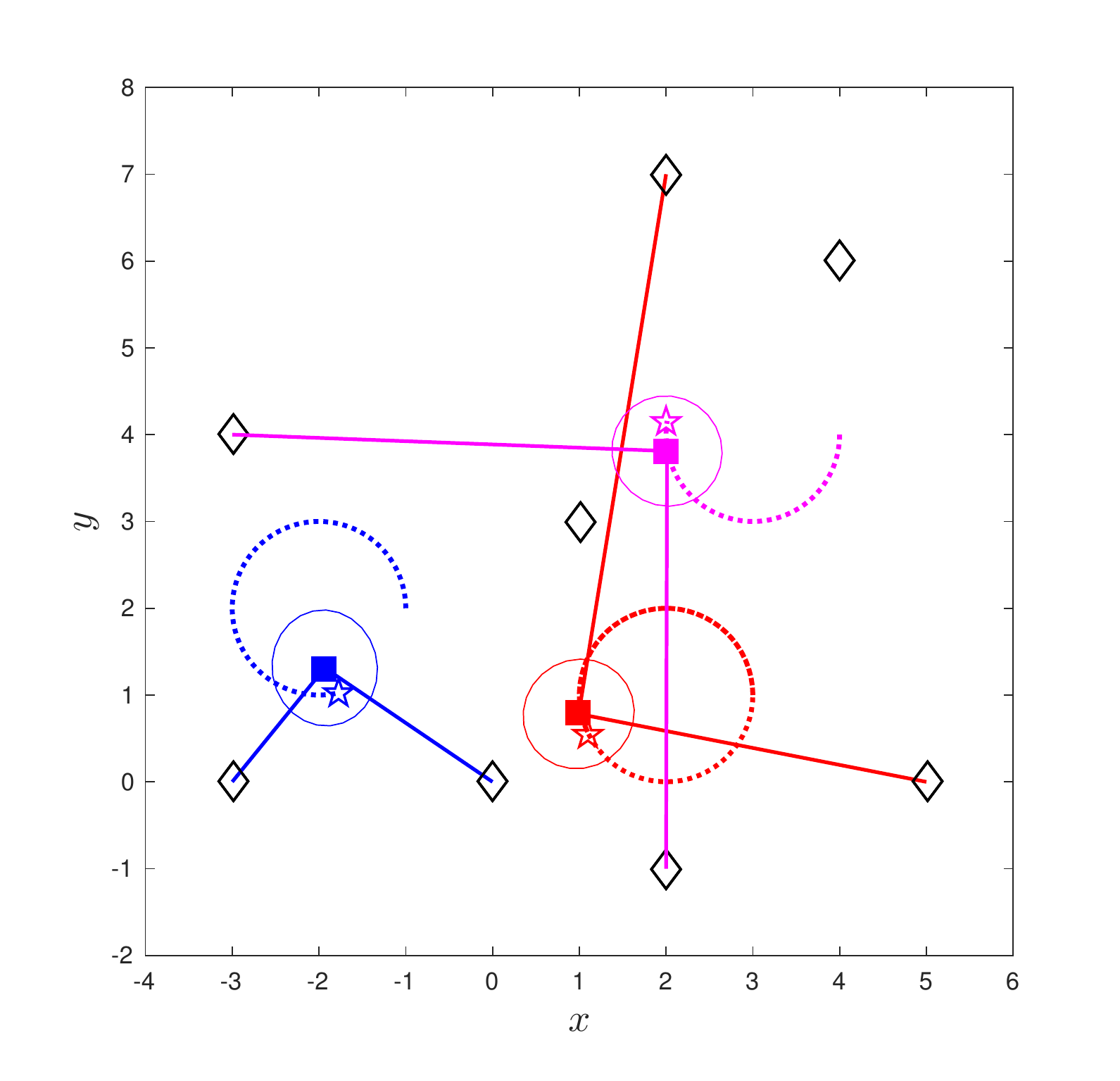}}
}
\caption{Greedy \revo{Non-overlapping} Pair Assignment (Algorithm~\ref{algorithm:unique_pair_assignment}) in action for tracking three targets with circular motion by using $\underline{C}^{-1}(O(p_{t_l},u_{t_l}))$. The three colors red, blue and magenta specify three targets, respectively. The pentagram, filled square, solid ellipse (sometimes, it looks like \revo{a} solid circle) and dotted circle indicates the true position, estimate mean position, variance and trajectory for the target, respectively. The black diamond indicates the sensor. The solid line joining the target and sensor indicates that the sensor is assigned to the target.\label{fig:uni_assign_tracking_invcond}}
\end{figure*}
% \begin{figure}[htb]
% \centering{
% \subfigure[]{
% \includegraphics[width=0.47\columnwidth]{figs/uni_assign_tar_mean_error.eps}}
% \subfigure[]{
% \includegraphics[width=0.47\columnwidth]{figs/uni_assign_tar_tra_cova}}
% }
% \caption{(a) The mean error for each target by Algorithm~\ref{algorithm:unique_pair_assignment} in Unique Pair Assignment within 100 time steps. (b) The trace of variance for each target by Algorithm~\ref{algorithm:unique_pair_assignment} in Unique Pair Assignment within 100 time steps.\label{fig:uni_err_trace}}
% \end{figure}
% \begin{figure}[htb]
% \centering{
% \subfigure[]{
% \includegraphics[width=0.47\columnwidth]{figs/uni_assign_tar_mean_error_submod.eps}}
% \subfigure[]{
% \includegraphics[width=0.47\columnwidth]{figs/uni_assign_tar_tra_cova_submod}}
% }
% \caption{(a) The mean error for each target by Algorithm~\ref{algorithm:unique_pair_assignment} in Unique Pair Assignment within 100 time steps. (b) The trace of variance for each target by Algorithm~\ref{algorithm:unique_pair_assignment} in Unique Pair Assignment within 100 time steps.\label{fig:uni_err_trace}}
% \end{figure}

\begin{figure}[htb]
\centering{
\subfigure[]{\includegraphics[width=0.9\columnwidth]{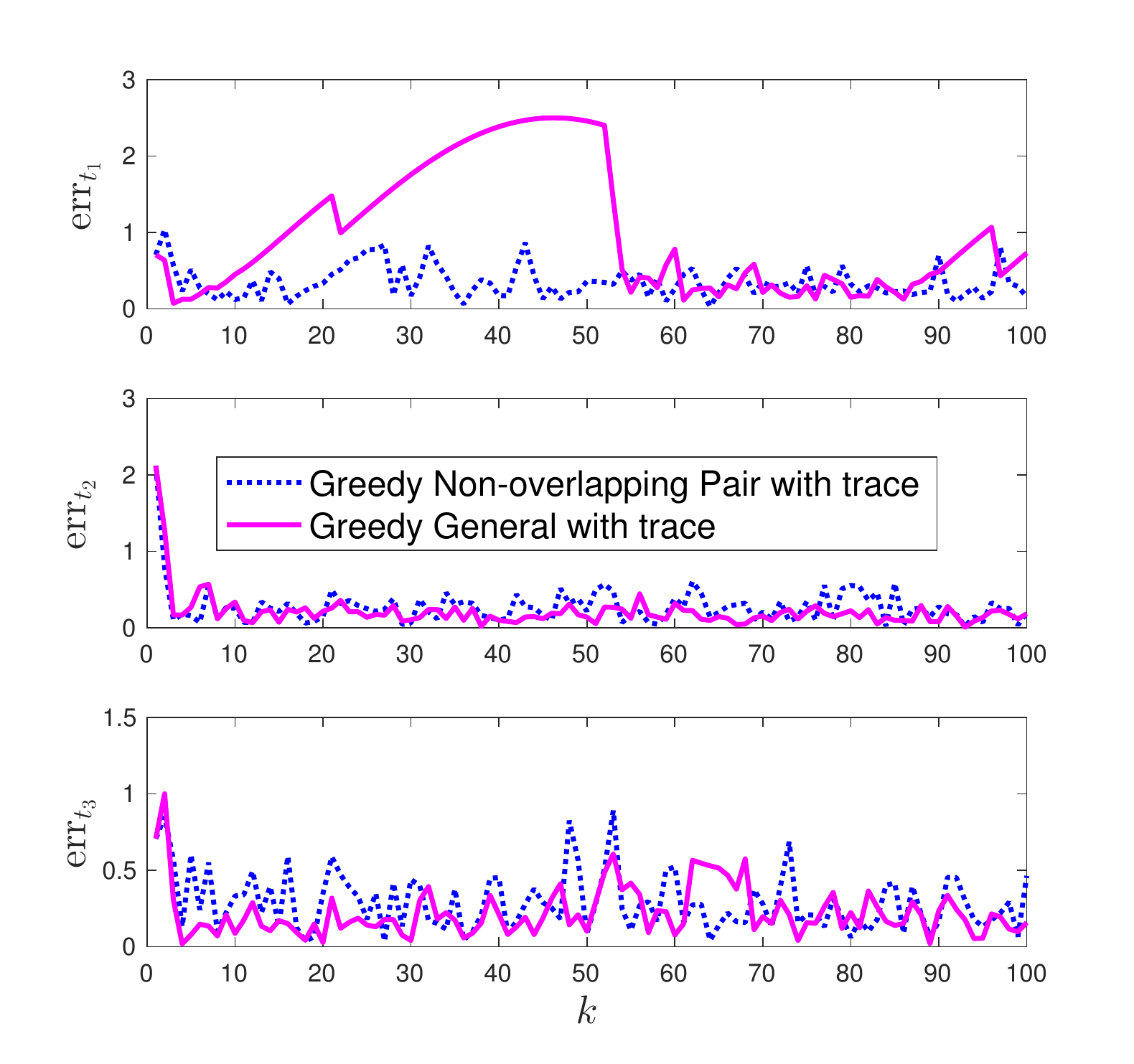}}
\subfigure[]{\includegraphics[width=0.9\columnwidth]{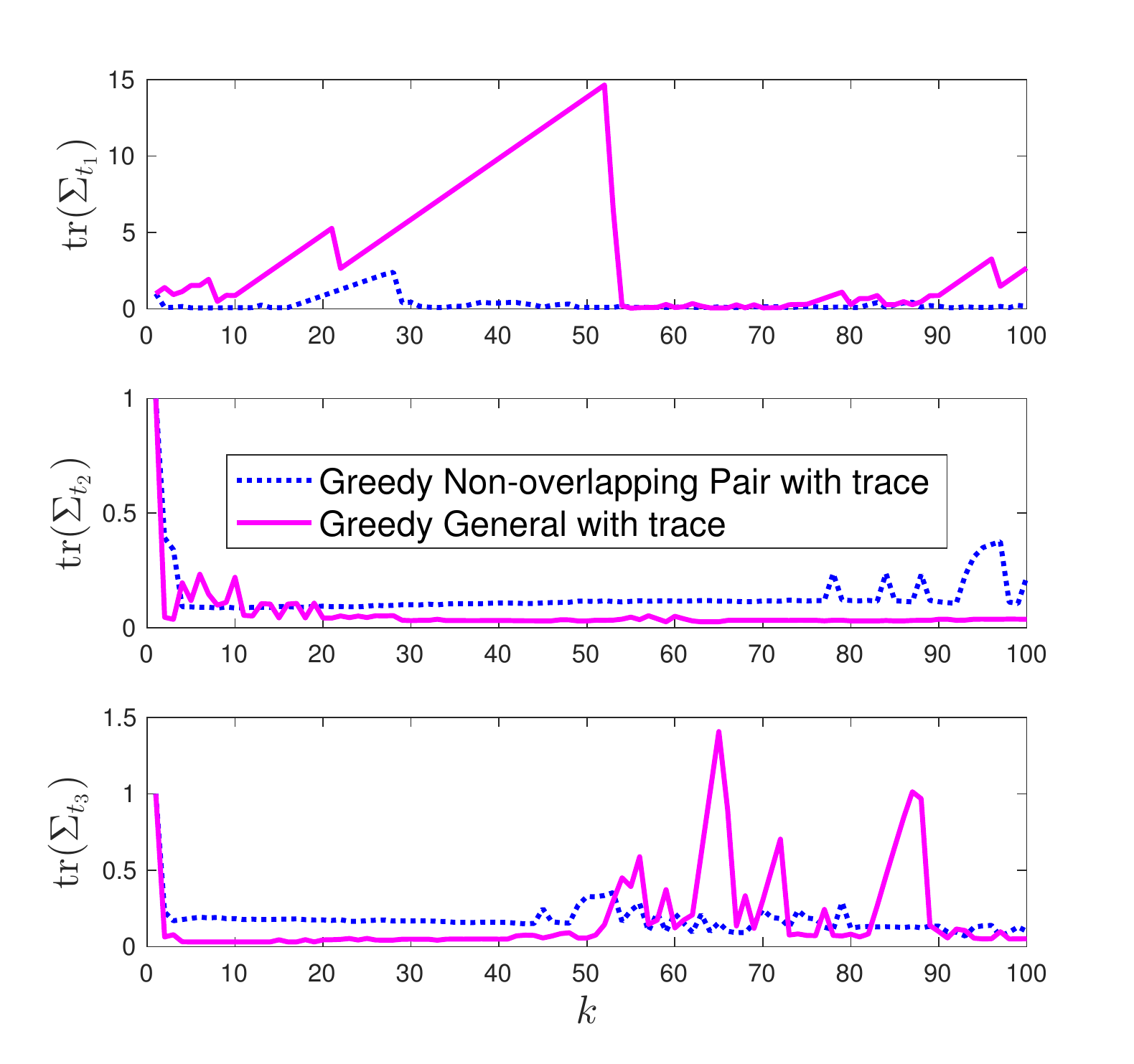}}
}
\caption{Comparison of mean error (a) and trace of covariance (b) for three-target tracking within 100 time steps by Greedy General Assignment~\cite{fisher1978analysis} with $\text{trace}(\mathbb{O}(p_{t_l}))$ and Greedy \revo{Non-overlapping} Pair Assignment (Algorithm~\ref{algorithm:unique_pair_assignment})  with $\text{trace}(\mathbb{O}(p_{t_l}))$. \label{fig:comparsion_gen_uni_log_mean_trace}} 
\end{figure}

% \begin{figure}[htb]
% \centering{
% \subfigure[]{\includegraphics[width=0.8\columnwidth]{figs/comp_uni_gen_log_mean.eps}}
% \subfigure[]{\includegraphics[width=0.8\columnwidth]{figs/comp_uni_gen_log_trace.eps}}
% }
% \caption{Comparison of mean error (a) and trace of covariance (b) for three-target tracking within 100 time steps by Greedy General Assignment~\cite{fisher1978analysis} with $\log(det(\mathbb{O}(p_{t_l},u_{t_l}))$ and Greedy Unique Pair Assignment (Algorithm~\ref{algorithm:unique_pair_assignment})  with $\log(det(\mathbb{O}(p_{t_l},u_{t_l})))$. \label{fig:comparsion_gen_uni_log_mean_trace}} 
% \end{figure}

\begin{figure}[htb]
\centering{
\subfigure[]{\includegraphics[width=0.9\columnwidth]{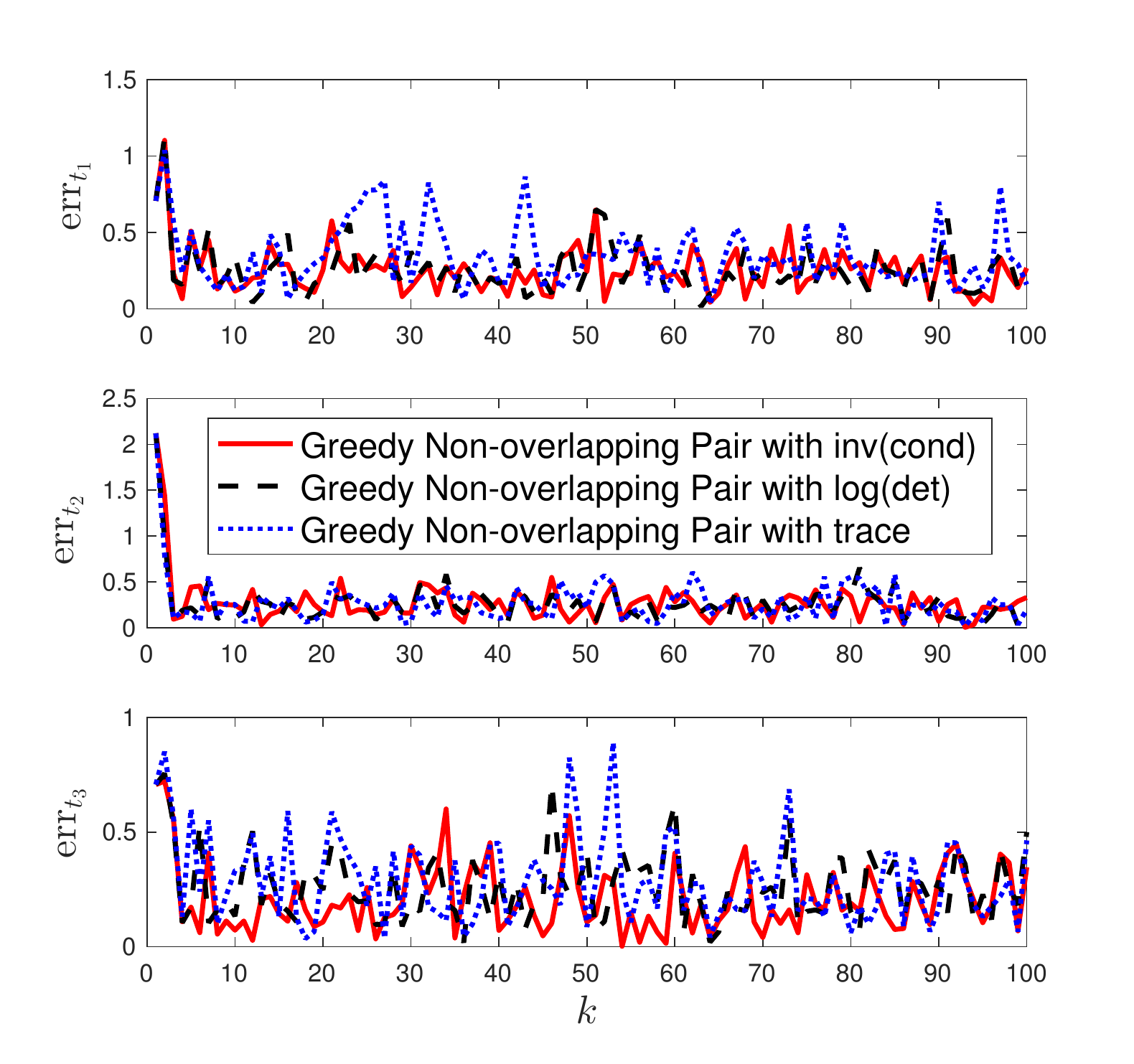}}
\subfigure[]{\includegraphics[width=0.9\columnwidth]{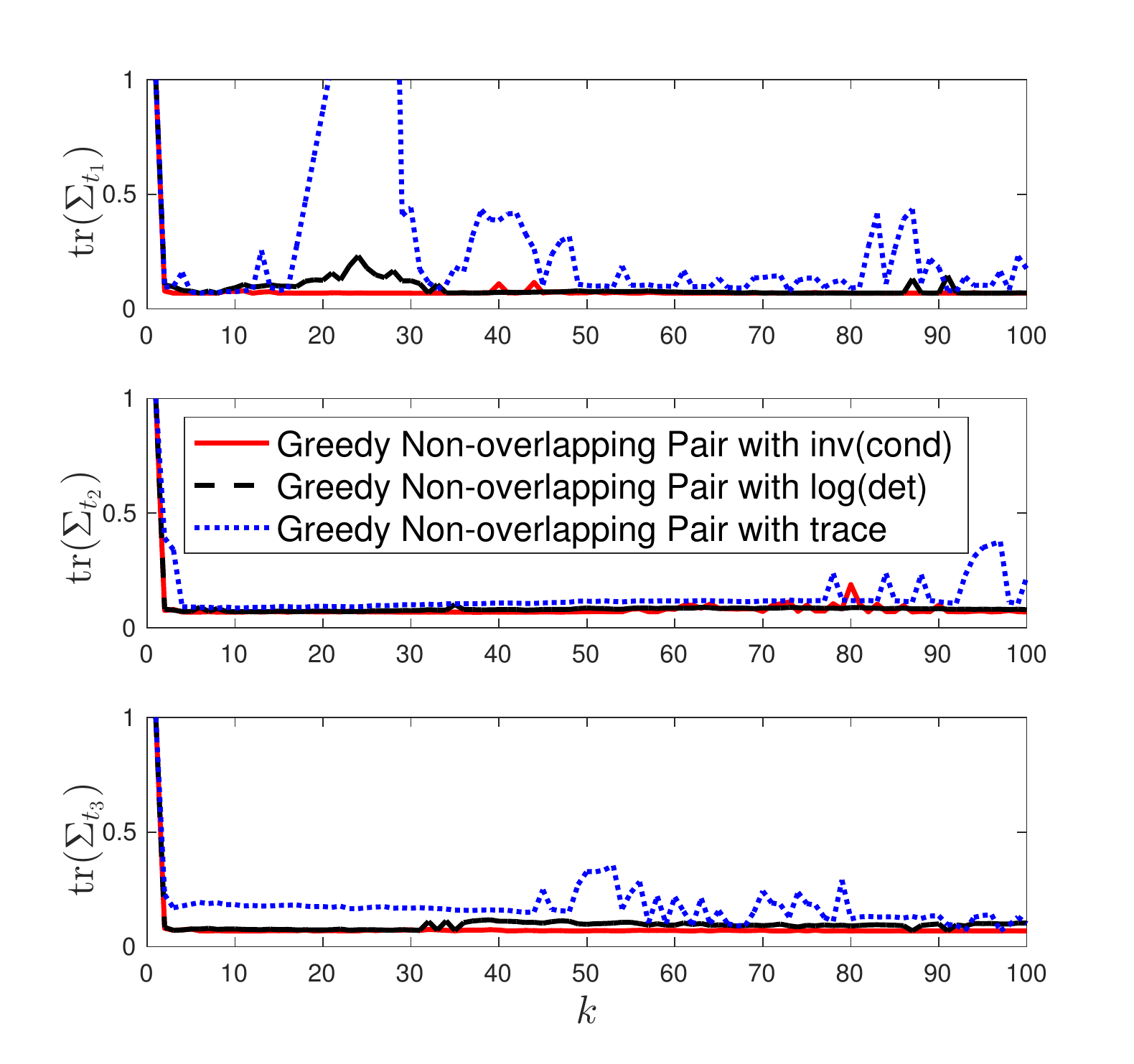}}
}
\caption{Comparison of mean error (a) and trace of covariance (b) for three-target tracking within 100 time steps by Greedy \revo{Non-overlapping} Pair Assignment (Algorithm~\ref{algorithm:unique_pair_assignment})  with $\underline{C}^{-1}(O(p_{t_l}))$, $\log\det(\mathbb{O}(p_{t_l}))$ and $\text{trace}(\mathbb{O}(p_{t_l}))$. \label{fig:comparsion_uni_logcondtrace_mean_trace}} 
\end{figure}

% \begin{figure}[htb]
% \centering{
% \subfigure[]{\includegraphics[width=0.8\columnwidth]{figs/comp_uni_invlog_mean.eps}}
% \subfigure[]{\includegraphics[width=0.8\columnwidth]{figs/comp_uni_invlog_trace.eps}}
% }
% \caption{Comparison of mean error (a) and trace of covariance (b) for three-target tracking within 100 time steps by Greedy Unique Pair Assignment (Algorithm~\ref{algorithm:unique_pair_assignment})  with $\underline{C}^{-1}(O(p_{t_l},u_{t_l}))$ and Greedy Unique Pair Assignment (Algorithm~\ref{algorithm:unique_pair_assignment}) with $\log(det(\mathbb{O}(p_{t_l},u_{t_l})))$. \label{fig:comparsion_uni_logcond_mean_trace}} 
% \end{figure}

\revo{\subsection{Baseline Comparisons}}
The \revo{Non-overlapping} Pair Assignment (Problem~\ref{prob:unique}) is NP-Complete. Therefore, finding $\omega(\textrm{OPT})$ is infeasible in polynomial time. In order to empirically evaluate the Greedy \revo{Non-overlapping} Pair Assignment (Algorithm~\ref{algorithm:unique_pair_assignment}), \revo{we use two baselines. When the number of sensors and targets is small, we compute the optimal solution using brute-force. When the number of sensors and targets is large, we compute an upper bound on the optimal solution value by solving a relaxed version of Problem~\ref{prob:unique}. In this relaxed version, one sensor is allowed to be assigned to multiple targets (unlike Problem~\ref{prob:unique}). However, we still require a pair of sensors to be assigned to at most one specific target.} 

We formulate the new assignment as Relaxed Pair Assignment (Problem \ref{prob:perfect}). It is clear that solving Relaxed Pair Assignment problem optimally gives us an upper bound of the optimality for \revo{Non-overlapping} Pair Assignment problem. We can use this upper bound for the comparison of the greedy approach in \revo{Non-overlapping} Pair Assignment.
\begin{problem}[Relaxed Pair Assignment] Given a set of sensors, $\mathcal{S} := \{s_0,\ldots,s_N\}$ and a set of targets, $\mathcal{T} := \{t_0,\ldots,t_L\}$, find an assignment of \revo{non-repetition} pairs of sensors to targets:
\begin{equation}
\text{maximize} \sum_{l=1}^L \omega(\sigma_1(t_l),\sigma_2(t_l), t_l)
\end{equation}
with the added constraint that all pairs are \revo{non-repetition}, that is, $\forall k, l= 1,\ldots, L$, $k\neq l$, $\sigma_1(t_k) \neq \sigma_1(t_l)$ \revo{or} $\sigma_2(t_k) \neq \sigma_2(t_l)$. 
\label{prob:perfect}
\end{problem}

The Relaxed Pair Assignment can be solved optimally by using maximum weight perfect bipartite matching (MWPBM)~\cite{cormen2009introduction}. %and solved for the optimal solution (Algorithm \ref{algorithm:maximum_perfect_matching}). We firstly formulate this problem as a Maximum Weight Perfect Bipartite Matching Problem:
Note that a sensor can be matched in multiple distinct pairs and assigned to multiple targets. \revo{This violates the constraint in Problem~\ref{prob:unique} where each sensor can be matched to at most one pair and assigned at most once. } The MWPBM can be solved using the Hungarian algorithm~\cite{kuhn1955hungarian} in polynomial time. 
% Figure \ref{fig:maximum_weighted_matching} shows an instance with six sensors and five targets. For each target, the upper bound of  the control input is specified as $u_{o_{l},max}=1$. We use the implementation for the matching provided online~\cite{matching}. 
% \begin{figure}[htb]
% \centering{
% %\subfigure[]{\includegraphics[width=0.8\columnwidth]{figs/multi_robot_target.eps}}
% \includegraphics[width=0.8\columnwidth]{figs/multi_robot_multi_target_matching.eps}
% }
% \caption{Maximum weighted perfect matching with six sensors $s_1$--$s_6$ and five targets $t_1$--$t_5$ at one timestep.\label{fig:maximum_weighted_matching}} 
% \end{figure}

\revo{While the Relaxed Pair Assignment computes an upper bound for $\omega(\textrm{OPT})$, we can compute $\omega(\textrm{OPT})$ exactly using brute-force when $N$ and $L$ are small  by enumerating all the possibilities. There are $\prod_{l=0}^{L-1} \binom{N-2l}{2}$ possible cases. Thus, the brute-force algorithm has an exponential running time. 

Figure~\ref{fig:comparsion_perfect_unique_bf} shows the total value of the greedy algorithm, $\omega(\text{GREEDY})$, the brute-force algorithm, $\omega(\text{OPT})$, and the MWPBM, $\omega(\text{MWPBM})$ for log det and inverse condition number. We simulate the following environment: $N=2L$, the positions of sensors and targets are generated randomly within $[0,100]\times [0,100]\in\mathbb{R}^{2}$ for 30 trials for each $L$, and the target's maximum control input is $u_{o,\max}=1$.  

We run the comparison code on a MacBook Pro with 2.6 GHz Intel Core i5 and  8 GB Memory. When $L=7$ and $N=14$, MATLAB runs out-of-memory when running the brute-force algorithm (there are $681080400$ possible cases for each trial). When $L=6$ and $N=12$, brute-force could not finish after running for 25 hours. Thus, we only consider the case when $L$ is varied from $1$ to $5$. 

%The code with log det measure and with inverse condition number runs takes approximately $2$ hours and $15$ minutes. Notably, without the brute-force algorithm, the codes with both measures run for several minutes only. 
From Figure~\ref{fig:comparsion_perfect_unique_bf}, we observe that $\omega(\text{MWPBM})$ is the highest because the MWPBM gives the upper bound of  $\omega(\text{OPT})$. More importantly, $\omega(\text{GREEDY})$ is close to $\omega(\text{OPT})$ and much higher than the theoretical bound of $\frac{1}{3}\omega(\textrm{OPT})$. }

\revo{Next, we compare $\omega(\textrm{GREEDY})$ and $\omega(\textrm{MWPBM})$, without brute-force, for larger values of $L$ and $N$. We vary $L$ from 1 to 20 and set $N=2L$.}  For both observability measures, Figure~\ref{fig:comparsion_perfect_unique} shows that $\omega(\textrm{GREEDY})$ is close to $\omega(\textrm{MWPBM})$ and much higher than $\frac{1}{3}\omega(\textrm{MWPBM})$. Thus, even though we give a theoretical $1/3$--approximation for the greedy algorithm, it performs much better in practice.

% \begin{figure}[htb]
% \centering{
% %\subfigure[]{\includegraphics[width=0.8\columnwidth]{figs/multi_robot_target.eps}}
% \includegraphics[width=0.9\columnwidth]{figs/comparsion_perfect_unique.eps}
% }
% \caption{Comparison of total reward collected by greedy approach (Algorithm~\ref{algorithm:unique_pair_assignment}) with the maximum perfect pair matching by using the lower bound of the inverse of the condition number of the  observability matrix as the measure.\label{fig:comparsion_perfect_unique}} 
% \end{figure}

\begin{figure}[htb]
\centering{
\subfigure[]{\includegraphics[width=0.8\columnwidth]{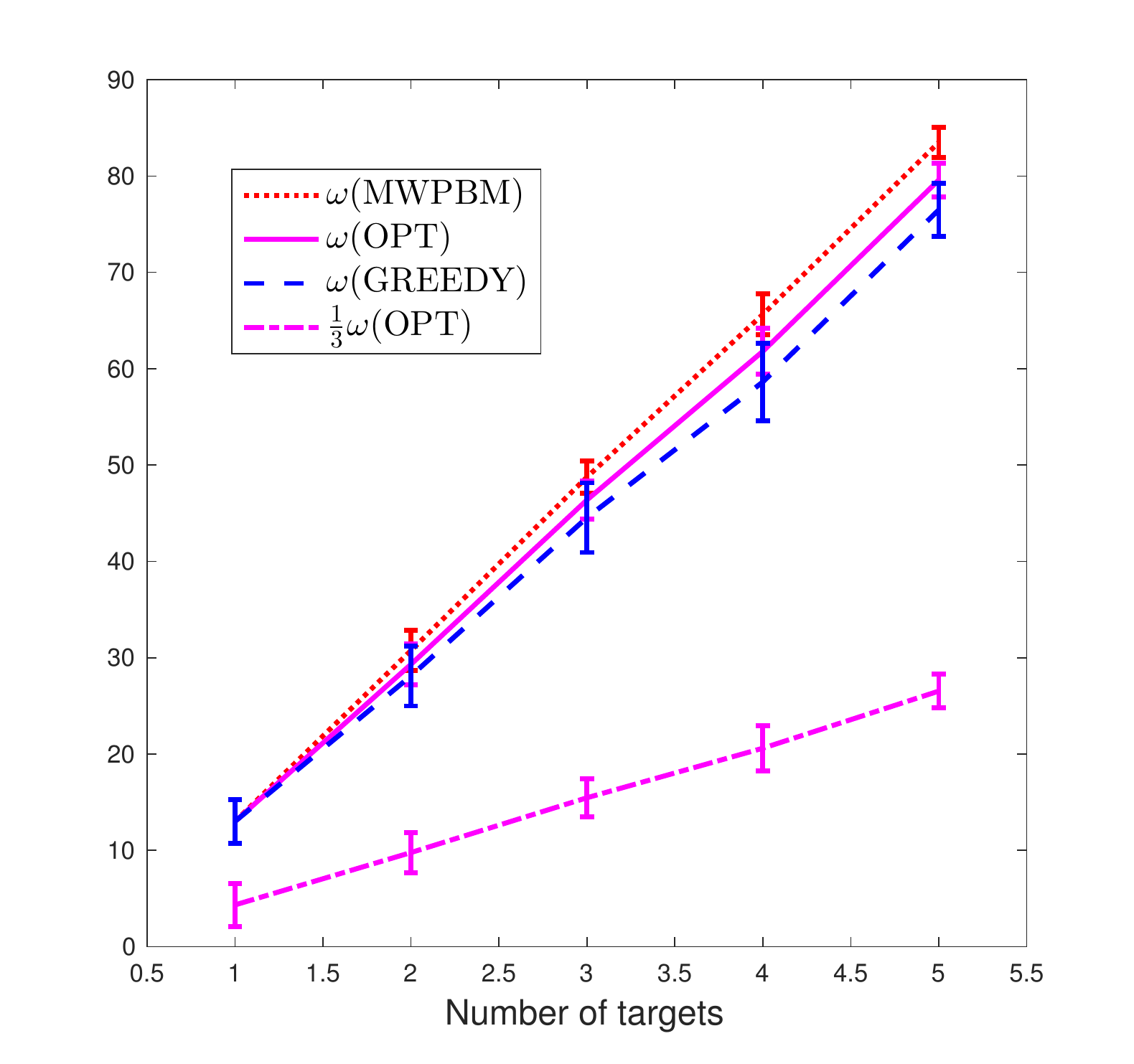}}
\subfigure[]{\includegraphics[width=0.8\columnwidth]{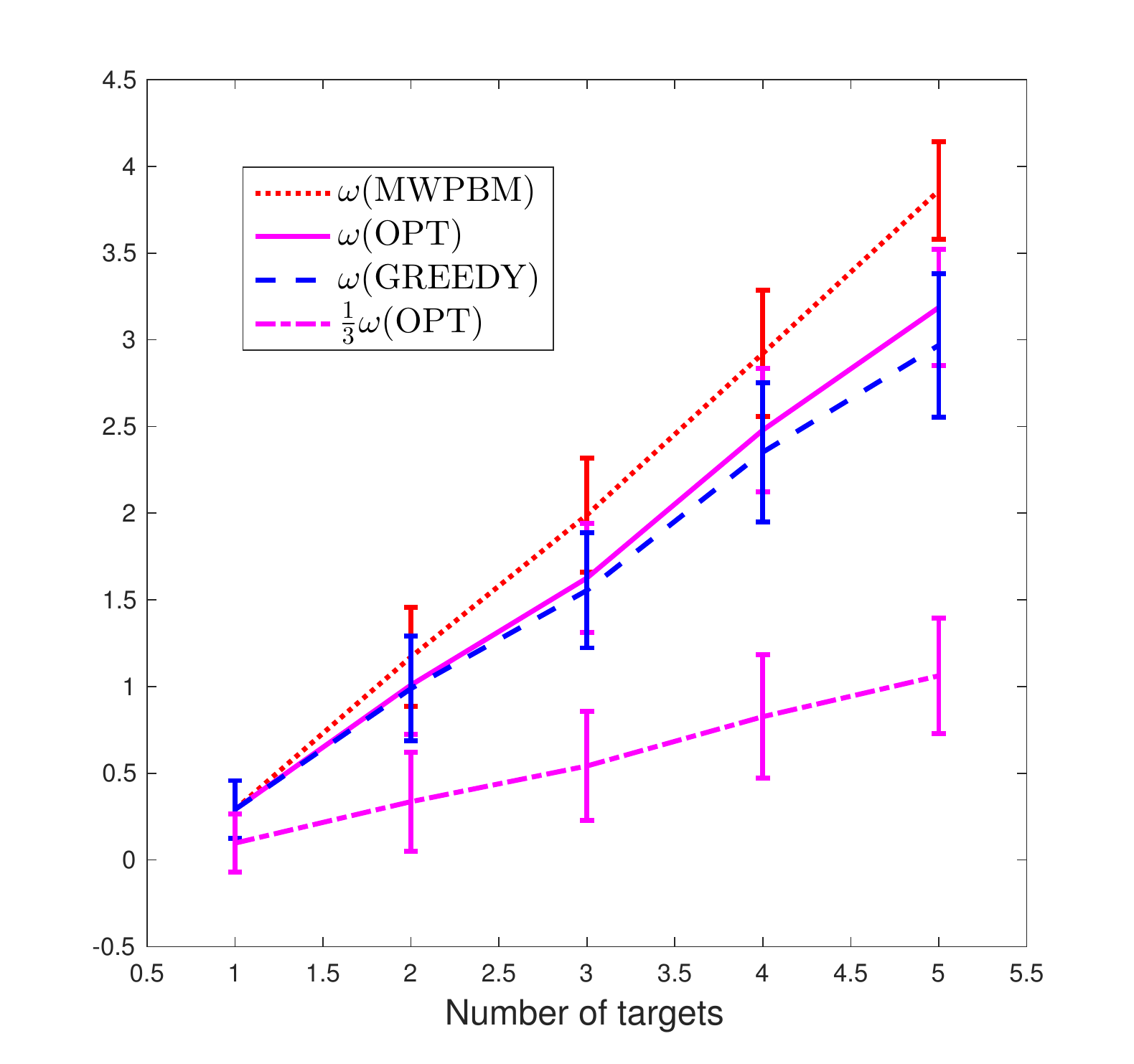}}
}
\caption{Comparison of \rev{the total value charged} by greedy approach (Algorithm~\ref{algorithm:unique_pair_assignment}) with the brute-force algorithm and the maximum perfect pair matching by using log det $\log(\det(\mathbb{O}(p_{t_l})))$ (a) and inverse condition number $\underline{C}^{-1}(O(p_{t_l}, u_{t_l}))$ (b), respectively.\label{fig:comparsion_perfect_unique_bf}} 
\end{figure}

\begin{figure}[htb]
\centering{
\subfigure[]{\includegraphics[width=0.8\columnwidth]{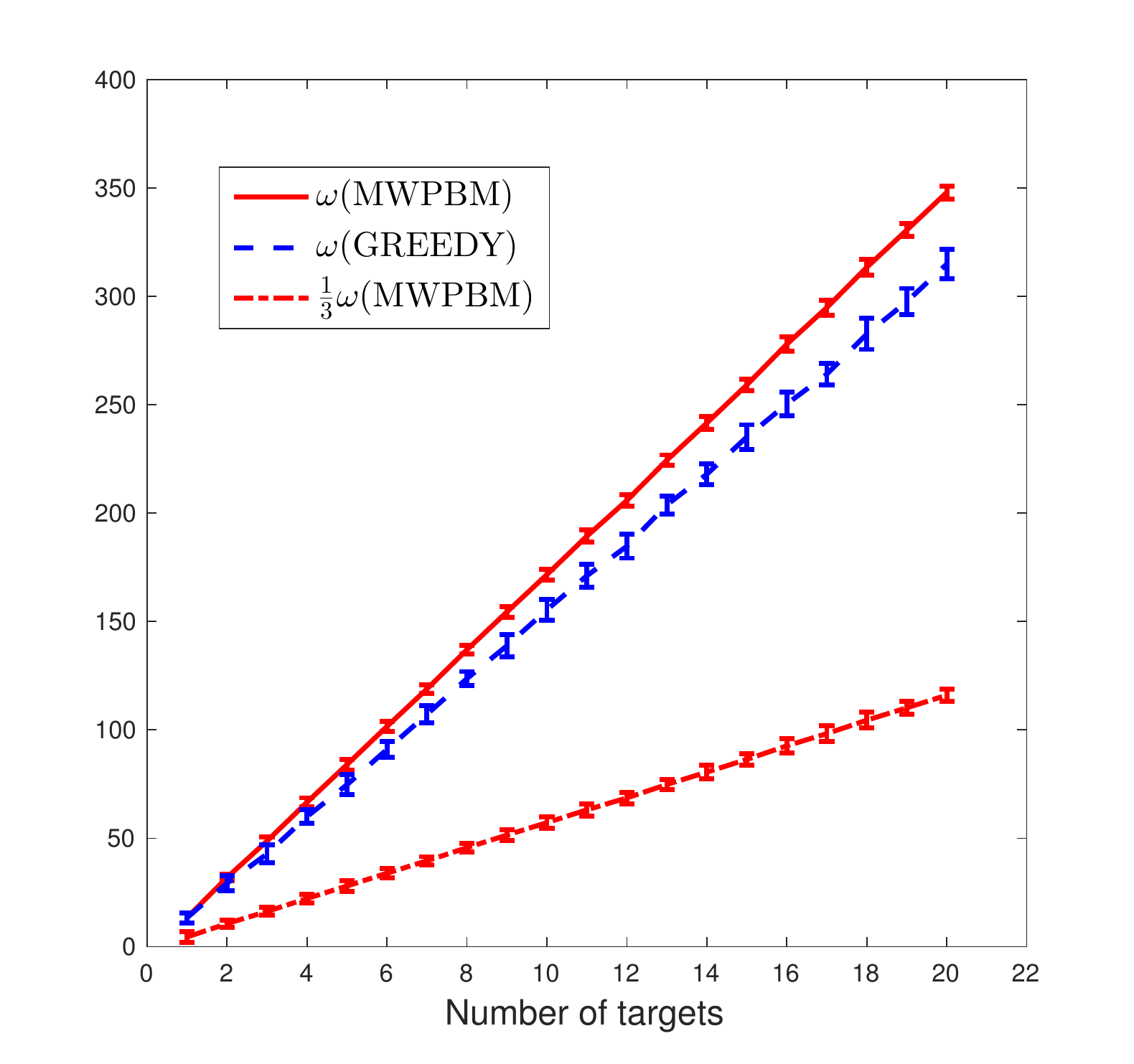}}
\subfigure[]{\includegraphics[width=0.8\columnwidth]{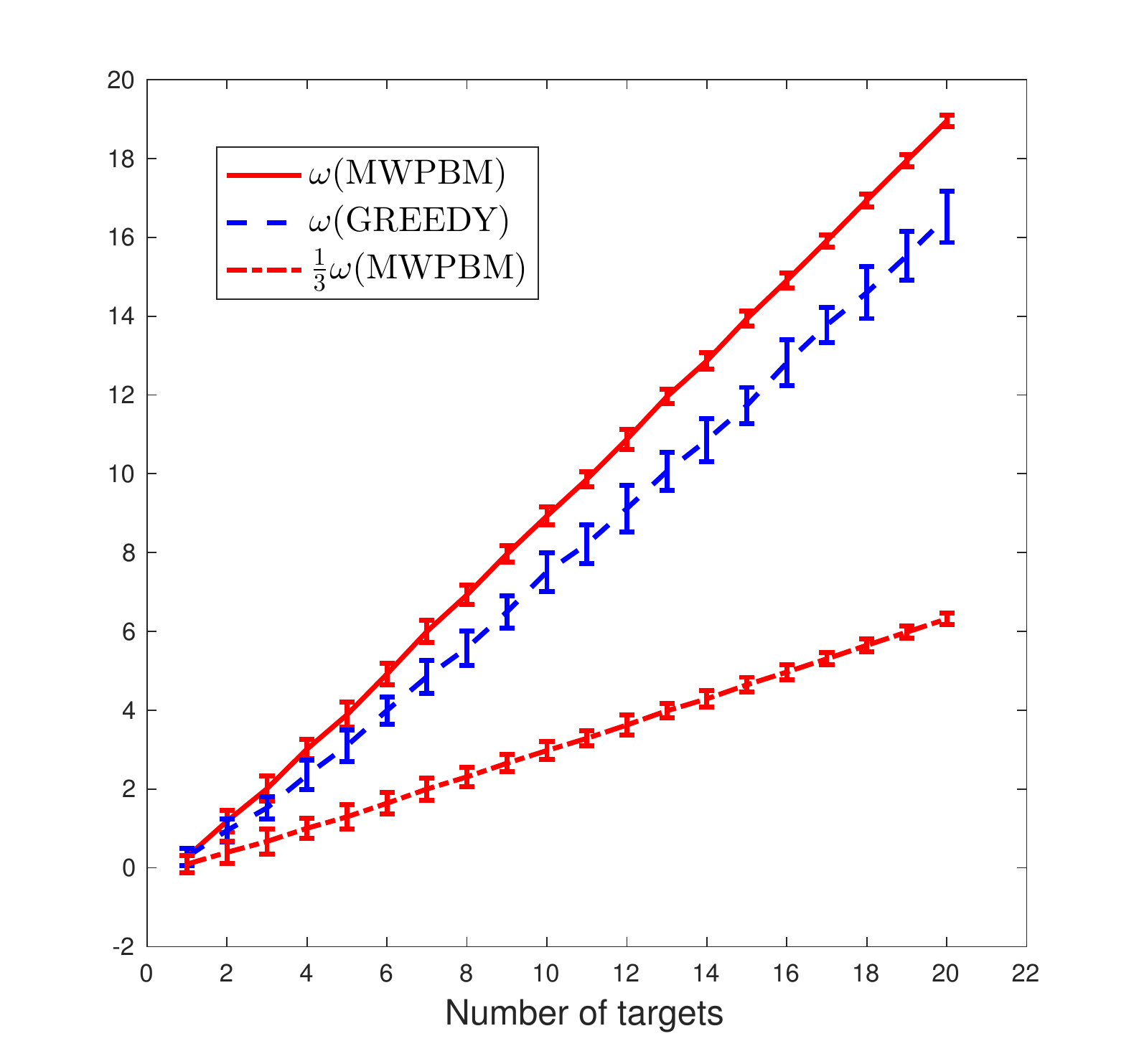}}
}
\caption{Comparison of \rev{the total value charged} by greedy approach (Algorithm~\ref{algorithm:unique_pair_assignment}) with the maximum perfect pair matching by using log det $\log(\det(\mathbb{O}(p_{t_l})))$ (a) and inverse condition number $\underline{C}^{-1}(O(p_{t_l}, u_{t_l}))$ (b), respectively.\label{fig:comparsion_perfect_unique}} 
\end{figure}

\section{Conclusion}
\label{sec:Conclusion}
In this paper, we solved sensor assignment problems to improve the observability for target tracking. We derived the lower bound on the inverse of the condition number of the observability matrix for a system with a mobile target and $N$ stationary sensors. The lower bound considers only the known part of the observability matrix --- the sensor-target relative position and an upper bound on the target's speed. We showed how this lower bound can be employed for sensor selection. % In particular, we showed how improving the lower bound by selecting the optimal set of sensors improves the tracking performance. We also simulated an adversarial case where the target moves to minimize observability and the sensors choose pairs to maximize observability. The salient feature of this work is that the target is powerful since it knows the actual observability matrix whereas the sensors only have a lower bound. Nevertheless, we observe that the sensors are able to track the target well. 
We considered two sensor assignment problems for which we presented constant-factor approximation algorithms. While we presented the two algorithms as alternatives to each other, they can also be combined to give better results. \rev{We proved the log det, $\log\det(\mathbb{O}(p_{t_l}))$ is submodular and monotone only when $\mathbb{O}(p_{t_l})$ is non-singular and performs better in terms of tracking than the trace (which is submodular and monotone). We can combine the greedy general assignment with the greedy \revo{non-overlapping} pair assignment to use log det as the observability measure. First, we use the \revo{non-overlapping} assignment to assign a pair of sensors to each target to make sure that $\mathbb{O}(p_{t_l})$ is non-singular. Then, we can improve the assignment strategy by using the greedy general assignment strategy to assign more sensors to each target.} 

Our immediate work is focused on assigning sensors to cover an area instead of tracking a group of targets. Another avenue is designing an efficient set covering strategy based on observability measures which are submodular and monotone. \revo{In many scenarios, the sensors are actually robots that are mobile~\cite{spletzer2003dynamic,martinez2006optimal,hollinger2009efficient,chung2006decentralized,huang2015bank,zhou2017active,zhou2018active,khodayi2018distributed}. in such cases, in addition to solving the assignment problem, we can also optimize the trajectories of the sensors. An immediate future work is to devise joint assignment and planning algorithms for better observability. }

Many of the results presented can be easily extended to 3D tracking. We conjecture that at least three sensors will be required for ensuring non-trivial lower bounds of the inverse condition number in 3D. The \revo{non-overlapping} pair assignment can be extended to assigning \revo{non-overlapping} triplets. We conjecture that this algorithm will yield a $1/4$-approximation.
%\section*{Acknowledgments}

%%%% appendix
\appendix
\section*{Proof for Theorem~\ref{thm:lower_bound}}
\begin{proof}
%Here, since $u_o$ is fixed, unknown and not controllable by the sensor, and thus $O(u_o)$ is not controllable by the sensor, either.  We focus on the relative state contribution part $O(o)$. Observability of trajectories (i.e., with control inputs) is not trivial (see for example \cite{gadre2004toward} for one sensor and one target case) and is the focus of our ongoing research.
The singular values of $O(p_{t_l},u_{t_l})$ can be found as the square-root of the eigenvalues of the \emph{symmetric observability matrix},  $\mathbb{O}(p_{t_l},u_{t_l})$, given as~\cite{Strang1976Linear},
\begin{eqnarray}
\mathbb{O}(p_{t_l},u_{t_l}) &=&O^{T}(p_{t_l},u_{t_l})O(p_{t_l},u_{t_l})\nonumber\\
&=&O^{T}(p_{t_l})O(p_{t_l})+O^{T}(u_{t_l})O(u_{t_l}).
\label{sy_ob_matrix}
\end{eqnarray}
\begin{equation}
\left\{
                \begin{array}{ll}
                  \sqrt{{\lambda_{\min}(\mathbb{O}(p_{t_l},u_{t_l}))}}={\sigma_{\min}(O(p_{t_l},u_{t_l}))},\\
                 \sqrt{{\lambda_{\max}(\mathbb{O}(p_{t_l},u_{t_l}))}}={\sigma_{\max}(O(p_{t_l},u_{t_l}))}.
                \end{array}
              \right.
\label{eqn:min/max_singualr_observability_metrix}
\end{equation}

We can use \emph{Weyl and dual Weyl} inequalities to bound the singular values. For Hermitian matrices $X$ and $Y$ with $r$ eigenvalues written in increasing order $\lambda_1(X)\leq \lambda_2(X)\leq\ldots\leq\lambda_r(X)$ and $\lambda_1(Y)\leq \lambda_2(Y)\leq\ldots\leq\lambda_r(Y)$, respectively, the \emph{Weyl inequalities}~\cite{franklin2012matrix} is given by,
\begin{equation}
\lambda_{i+j-1}(X+Y)\geq \lambda_{i}(X)+\lambda_{j}(Y)
\label{eqn:Weyl_inequalities}
\end{equation}
where $i,j\geq 1$ and $i+j-1\leq r$. Similarly, the \emph{dual Weyl inequalities} is given by
\begin{equation}
\lambda_{i+j-r}(X+Y)\leq \lambda_{i}(X)+\lambda_{j}(Y)
\label{eqn:Weyl_inequalities}
\end{equation}
where $i,j\geq 1$ and $i+j-r\leq r$.

Since $\mathbb{O}(p_{t_l},u_{t_l})\in \mathbb{R}^{2\times 2}$, $O^{T}(p_{t_l})O(p_{t_l})\in \mathbb{R}^{2\times 2}$ and $O^{T}(u_{t_l})O(u_{t_l})\in \mathbb{R}^{2\times 2}$ are symmetric matrices, they are Hermitian with the eigenvalues (in ascending order) as $\lambda_1(\mathbb{O}(p_{t_l},u_{t_l}))\leq \lambda_2(\mathbb{O}(p_{t_l},u_{t_l}))$, $\lambda_1(O^{T}(p_{t_l})O(p_{t_l}))\leq \lambda_2(O^{T}(p_{t_l})O(p_{t_l}))$ and $\lambda_1(O^{T}(u_{t_l})O(u_{t_l}))\leq \lambda_2(O^{T}(u_{t_l})O(u_{t_l}))$. Following the \emph{Weyl and dual Weyl} inequalities, we get
\begin{align}
&\lambda_1(\mathbb{O}(p_{t_l},u_{t_l})) \geq \lambda_1(O^{T}(p_{t_l})O(p_{t_l})) + \lambda_1(O^{T}(u_{t_l})O(u_{t_l})),\nonumber\\
% &\lambda_1(\mathbb{O}(p_{t_l},u_{t_l})) \leq \nonumber\\ 
% &\min\left\{
%                 \begin{array}{ll}
%                  \lambda_1(O^{T}(p_{t_l})O(p_{t_l})) + \lambda_2(O^{T}(u_{t_l})O(u_{t_l}))\\
%                   \lambda_2(O^{T}(p_{t_l})O(p_{t_l})) + \lambda_1(O^{T}(u_{t_l})O(u_{t_l}))
%                 \end{array},
%               \right. \nonumber\\
&\lambda_2(\mathbb{O}(p_{t_l},u_{t_l})) \leq
\lambda_2(O^{T}(p_{t_l})O(p_{t_l})) + \lambda_2(O^{T}(u_{t_l})O(u_{t_l})).
% &\lambda_2(\mathbb{O}(p_{t_l},u_{t_l}) \geq \nonumber\\
% & \max\left\{
%                 \begin{array}{ll}
%                  \lambda_1(O^{T}(p_{t_l})O(p_{t_l})) + \lambda_2(O^{T}(u_{t_l})O(u_{t_l}))\\
%                   \lambda_2(O^{T}(p_{t_l})O(p_{t_l})) + \lambda_1(O^{T}(u_{t_l})O(u_{t_l}))
%                 \end{array}.
%               \right.
\label{eqn:Weyl_eigen_four}
\end{align}
Thus,
\begin{align}
&\frac{\lambda_{1}(\mathbb{O}(p_{t_l},u_{t_l}))}{\lambda_{2}(\mathbb{O}(p_{t_l},u_{t_l}))}\geq
\nonumber\\
&\frac{\lambda_1(O^{T}(p_{t_l})O(p_{t_l})) + \lambda_1(O^{T}(u_{t_l})O(u_{t_l}))}{\lambda_2(O^{T}(p_{t_l})O(p_{t_l})) + \lambda_2(O^{T}(u_{t_l})O(u_{t_l}))}.
\label{eqn:lower_bound_inequality}
\end{align}
Then, from Equation \ref{eqn:min/max_singualr_observability_metrix} and Equation \ref{eqn:lower_bound_inequality},  the inverse of the condition number of the local nonlinear observability matrix, 
\begin{align}
\label{eqn:lower_bound_inversecond}
&C^{-1}(O(p_{t_l},u_{t_l}))=\sqrt{\frac{\lambda_{1}(\mathbb{O}(p_{t_l},u_{t_l}))}{\lambda_{2}(\mathbb{O}(p_{t_l},u_{t_l}))}}\nonumber\\
&\geq \sqrt{\frac{\lambda_1(O^{T}(p_{t_l})O(p_{t_l})) + \lambda_1(O^{T}(u_{t_l})O(u_{t_l}))}{\lambda_2(O^{T}(p_{t_l})O(p_{t_l})) + \lambda_2(O^{T}(u_{t_l})O(u_{t_l}))}}.
\end{align}
\revo{Thus, we have the lower of the inverse of the condition number, \begin{equation}
    \sqrt{\frac{\lambda_1(O^{T}(p_{t_l})O(p_{t_l})) + \lambda_1(O^{T}(u_{t_l})O(u_{t_l}))}{\lambda_2(O^{T}(p_{t_l})O(p_{t_l})) + \lambda_2(O^{T}(u_{t_l})O(u_{t_l}))}}.
    \label{eqn:lowerbound_inverse}
\end{equation}}
By calculating the eigenvalues of symmetric matrix of target's control contribution,
$$O^{T}(u_{t_l})O(u_{t_l})=\begin{bmatrix}
     u_{lx}^{2} & u_{lx}u_{ly}\\
    u_{ly}u_{lx} & u_{ly}^{2}
\end{bmatrix},$$
we have, 
\begin{eqnarray}
&&\lambda_{1}(O^{T}(u_{t_l})O(u_{t_l}))=0\nonumber\\
&&\lambda_{2}(O^{T}(u_{t_l})O(u_{t_l}))=u_{lx}^{2}+u_{ly}^{2}=u_{t_l}^{2}.
\label{eqn:max_min_eigenu}
\end{eqnarray}
\noindent Then the lower bound of $C^{-1}(O(o,u_o))$ \revo{(Equation~\ref{eqn:lowerbound_inverse})} is calculated as  \begin{eqnarray}
%\label{eqn:lower_bound_inversecond}
\underline{C}^{-1}(O(p_{t_l},u_{t_l})) &=&  \sqrt{\frac{\lambda_1(O^{T}(p_{t_l})O(p_{t_l}))}{\lambda_2(O^{T}(p_{t_l})O(p_{t_l})) + u_{t_l}^{2}}}\nonumber\\
&=& \frac{\sigma_{\min}(O(p_{t_l}))}{\sqrt{\sigma^{2}_{\max}(O(p_{t_l}))+ u_{t_l}^{2}}}
\label{eqn:inverse_lower_bound_simply}
\end{eqnarray}
Equation~\ref{eqn:inverse_lower_bound_simply} gives the main lower bound. Note that $C^{-1}(O(p_{t_l},u_{t_l}))$ cannot be determined since target's control input, $u_{t_l}$, is unknown. However, we know that $||u_{t_l}||_2 \leq u_{t_l,\max}$. Therefore,
\begin{equation}
%\label{eqn:lower_bound_inversecond}
\underline{C}^{-1}(O(p_{t_l},u_{t_l})) \geq \frac{\sigma_{\min}(O(p_{t_l}))}{\sqrt{\sigma^{2}_{\max}(O(p_{t_l}))+ u_{t_l,\max}^{2}}}
\label{eqn:inverse_lower_bound_withumax}
\end{equation}
This yields our main lower bound result.
\label{proof:theorem1}
\end{proof}
\section*{Proof for \revo{Corollary~\ref{cor:sensor_n_1}}}
\begin{proof}
The local observability matrix for one-sensor-target, $s_i-t_l$ system can be derived from Equation \ref{eqn:new_multi_sensor_target} as,
\begin{equation}
O_{i}(p_{t_l},u_{t_l})=\begin{bmatrix}
    x_{t_l}-x_{s_i} & y_{t_l}-y_{s_i} \\ 
     u_{lx} & u_{ly}   
\end{bmatrix}.
\label{eqn:observability_i_o}
\end{equation}
The sensor-target relative state contribution is 
\begin{equation}
  O_{i}(p_{t_l})=\begin{bmatrix}
    x_{t_l}-x_{s_i} & y_{t_l}-y_{s_i}
\end{bmatrix}.
\label{eqn:sensor_tar_realtive}
\end{equation}
% It is easy to see that the single sensor cannot improve the lower bound of the observability matrix.
The $s_i-t_l$ system is weakly locally observable if $O_{i}(p_{t_l},u_{t_l})$ has full column rank, i.e., $(x_{t_l}-x_{s_i})u_{ly} \neq (y_{t_l}-y_{s_i})u_{lx}$. However, the sensor does not know the target's control input, $u_{t_l}$. %Observability of trajectories (i.e., with control inputs) is not trivial (see for example \cite{gadre2004toward} for one sensor and one target case) and is the focus of our ongoing research.

% both the rank and the degree of the observability is uncontrollable by the sensors based on the proofs of both \textbf{Theorem 1} and \textbf{Theorem 2}.
% Consider the number of sensors, $N=1$. The relative state observability matrix with one sensor $i$ and target $o$ can be described as 

% $$O_{i}(o)=\begin{bmatrix}
%     o_{x}-p_{ix} & o_{y}-p_{iy}
% \end{bmatrix},$$
% \noindent and
Given the symmetric matrix by sensor-target relative state contribution of $s_i-t_l$ system,
\begin{equation}
    O_{i}^{T}(p_{t_l})O_{i}(p_{t_l})=\begin{bmatrix}
    (x_{t_l}-x_{s_i})^{2},(x_{t_l}-x_{s_i})(y_{t_l}-y_{s_i}) \\ 
     (x_{t_l}-x_{s_i})(y_{t_l}-y_{s_i}), (y_{t_l}-y_{s_i})^{2} 
\end{bmatrix},
\end{equation}
we have, 
\begin{equation}
\left\{
                \begin{array}{ll}
                  {\sigma_{\min}(O_{i}(p_{t_l}))} =\sqrt{{\lambda_{\min}(O_{i}^{T}(p_{t_l})O_{i}(p_{t_l}))}}=0,\\
                 {\sigma_{\max}(O_{i}(p_{t_l}))} =\sqrt{{\lambda_{\max}(O_{i}^{T}(p_{t_l})O_{i}(p_{t_l}))}}\\
~~~~~~~~~~~~~~~=\sqrt{(x_{t_l}-x_{s_i})^{2}+(y_{t_l}-y_{s_i})^{2}}.
                \end{array}
              \right.
\label{eqn:min/max_singualr_observability_metrix_relative_state}
\end{equation}
Thus, from Equation \ref{eqn:inverse_lower_bound_simply}, the lower bound for $C^{-1}(O_{i}(p_{t_l},u_{t_l}))$ is $\frac{\sigma_{\min}(O_{i}(p_{t_l}))}{\sqrt{\sigma^{2}_{\max}(O_{i}(p_{t_l}))+ u_{t_l}^{2}}}=0$. Consequently, the lower bound cannot be controlled by the sensor. 
\end{proof}
\section*{Proof for \revo{Corollary~\ref{cor:sensor_n_gre2}}}
\begin{proof} According to Equation~\ref{eqn:min/max_singualr_observability_metrix}, we have,
\begin{equation}
\frac{\sigma_{\min}(O(p_{t_l},u_{t_l}))}{\sigma_{\max}(O(p_{t_l},u_{t_l}))}=\frac{\sqrt{\lambda_{\min}(\mathbb{O}(p_{t_l},u_{t_l}))}}{\sqrt{\lambda_{\max}(\mathbb{O}(p_{t_l},u_{t_l}))}}.
\label{eqn:eigensquare_observability_metrix}
\end{equation}
Following Equations~\ref{eqn:lower_bound_inequality} and \ref{eqn:max_min_eigenu}, we have the lower bound for $\frac{\lambda_{\min}(\mathbb{O}(p_{t_l},u_{t_l}))}{\lambda_{\max}(\mathbb{O}(p_{t_l},u_{t_l}))}$, described as
\begin{eqnarray}
\frac{\lambda_{\min}(\mathbb{O}(p_{t_l},u_{t_l}))}{\lambda_{\max}(\mathbb{O}(p_{t_l},u_{t_l}))}\geq \frac{\lambda_{\min}(O^{T}(p_{t_l})O(p_{t_l}))}{\lambda_{\max}(O^{T}(p_{t_l})O(p_{t_l})) + u_{t_l}^{2}},
% \frac{\lambda_{1}(\mathbb{O}(o,u_o))}{\lambda_{2}(\mathbb{O}(o,u_o))}\leq \frac{\min\left\{
%                 \begin{array}{ll}
%                  \lambda_1(O^{T}(o)O(o)) + \lambda_2(O^{T}(u_o)O(u_o))\\
%                   \lambda_2(O^{T}(o)O(o)) + \lambda_1(O^{T}(u_o)O(u_o))
%                 \end{array}
%               \right. }{\max\left\{
%                 \begin{array}{ll}
%                  \lambda_1(O^{T}(o)O(o)) + \lambda_2(O^{T}(u_o)O(u_o))\\
%                   \lambda_2(O^{T}(o)O(o)) + \lambda_1(O^{T}(u_o)O(u_o))
%                 \end{array}
%               \right.}.
\label{eqn:upper_lower_bounds}
\end{eqnarray}
Now, we equivalently rewrite the statement of the theorem (using eigenvalues instead of singular values) as:
\emph{if} 
\begin{equation}
\left\{
                \begin{array}{ll}
                 \frac{\lambda_{\min}(\revo{O^{'}}^T(p_{t_l})\revo{O^{'}}(p_{t_l}))}{\lambda_{\max}(\revo{O^{'}}^{T}(p_{t_l})\revo{O^{'}}(p_{t_l}))}\geq\frac{\lambda_{\min}(O^{T}(p_{t_l})O(p_{t_l}))}{\lambda_{\max}(O^{T}(p_{t_l})O(p_{t_l}))},\\
                \lambda_{\min}(\revo{O^{'}}^{T}(p_{t_l})\revo{O^{'}}(p_{t_l}))\geq \lambda_{\revo{\min}}(O^{T}(p_{t_l})O(p_{t_l}))
                \end{array}
              \right.
\label{eqn:assumption_singular_condition}
\end{equation}
\emph{then} 
\begin{eqnarray}
&&\frac{\lambda_{\min}(\revo{O^{'}}^{T}(p_{t_l})\revo{O^{'}}(p_{t_l}))}{\lambda_{\max}(\revo{O^{'}}^{T}(p_{t_l})\revo{O^{'}}(p_{t_l})) + u_{t_l}^{2}} \nonumber\\
&&\geq \frac{\lambda_{\min}(O^{T}(p_{t_l})O(p_{t_l}))}{\lambda_{\max}(O^{T}(p_{t_l})O(p_{t_l})) + u_{t_l}^{2}}
\label{eqn:conslusion_singular_condition}
\end{eqnarray}
where the $\lambda$ and $\lambda'$ denotes the eigenvalues before and after the sensors increase $C^{-1}(O(p_{t_l}))$ and $\sigma_{\min}(O(p_{t_l})$, respectively.

We start with the left-hand side of Equation~\ref{eqn:conslusion_singular_condition} to get,
\begin{eqnarray}
&&\frac{\lambda_{\min,l}^{'}}{\lambda_{\max,l}^{'} + u_{t_l}^{2}}-\frac{\lambda_{\min,l}}{\lambda_{\max,l} + u_{t_l}^{2}}\nonumber\\
&=&\frac{\lambda_{\min,l}^{'}\lambda_{\max,l}-\lambda_{\max,l}^{'}\lambda_{\min,l}}{(\lambda_{\max,l}^{'} + u_{t_l}^{2})(\lambda_{\max,l} + u_{t_l}^{2})}\nonumber\\
&+&\frac{u_{t_l}^{2}(\lambda_{\min,l}^{'}-\lambda_{\min,l})}{(\lambda_{\max,l}^{'} + u_{t_l}^{2})(\lambda_{\max,l} + u_{t_l}^{2})}
\label{eqn:conclude_proof}
\end{eqnarray}
where $\lambda_{\min,l}^{'}$, $\lambda_{\max,l}^{'}$, $\lambda_{\min,l}$ and $\lambda_{\max,l}$ denote the simplified forms of $\lambda_{\min}(\revo{O^{'}}^{T}(p_{t_l})\revo{O^{'}}(p_{t_l}))$, $\lambda_{\max}(\revo{O^{'}}^{T}(p_{t_l})\revo{O^{'}}(p_{t_l}))$, $\lambda_{\min}(O^{T}(p_{t_l})O(p_{t_l}))$ and $\lambda_{\max}(O^{T}(p_{t_l})O(p_{t_l}))$, respectively. \revo{Note that Equation~\ref{eqn:assumption_singular_condition} can be reordered to match the numerator in the
second line of Equation~\ref{eqn:conclude_proof}. Then we have,} 
\begin{equation}
\frac{\lambda_{\min,l}^{'}}{\lambda_{\max,l}^{'} + u_{t_l}^{2}}-\frac{\lambda_{\min,l}}{\lambda_{\max,l} + u_{t_l}^{2}} \geq 0.
\label{eqn:conclusion_proof}
\end{equation}
Hence the claim is proved.
\end{proof}
\begin{rem}
Equation~\ref{eqn:assumption_singular_condition} is sufficient, but not necessary condition to guarantee Equation~\ref{eqn:conslusion_singular_condition}. This is because Equation~\ref{eqn:conslusion_singular_condition} can be established with a weaker condition, 
\begin{equation*}
\lambda_{\min,l}^{'}\lambda_{\max,l}-\lambda_{\max,l}^{'}\lambda_{\min,l}+u_{t_l}^{2}(\lambda_{\min,l}^{'}-\lambda_{\min,l})\geq 0.
\label{eqn:weaker_assumption*}
\end{equation*}
We choose the stricter condition (Equation~\ref{eqn:assumption_singular_condition}) because it is conceptually easy to separate and eliminate the influence on the degree of observability from target's control input, $u_{t_l}$, which is unknown and uncontrolled.
\label{remark:rem2}
\end{rem}
\rev{
\section*{Proof for Theorem~\ref{thm:submodular_observability_measure}}
We first give the proofs of the properties of the \emph{symmetric observability matrix by sensor-target relative state contribution}, $\mathbb{O}(p_{t_l})$. Then we obtain the similar result for the \emph{symmetric observability matrix}, $\mathbb{O}(p_{t_l}, u_{t_l})$ by a minor extension. For the target $t_l$, denote its sensor-target relative state vector as $r_{il} = [x_{tl}-x_{si}, y_{tl}-y_{si}], i\in\{1,2,\cdots, N\}$. Denote its ground sensor-target relative state set as $\mathcal{V}=\{r_{1l}, r_{2l}, \cdots, r_{Nl}\}$. For a given $\mathcal{W} \subseteq \mathcal{V}$, we form $R_{\mathcal{W}}=[R_{0l}, r_{wl}]^{T}$ with a (possibly empty) existing matrix $R_{0l}$ and the associated sensor-target relative vector $r_{wl}\in \mathcal{W}$. We calculate the associated symmetric observability matrix by sensor-target relative state contribution as $\mathbb{O}(p_{t_l})_{\mathcal{W}} = R_{wl}^{T}R_{wl}$. To \revo{simplify} notation, we write $\mathbb{O}(p_{t_l})_{\mathcal{W}}$ as $\mathbb{O}_{\mathcal{W}}$.
\begin{lem}
The set function mapping subsets $\mathcal{W} \subseteq V$ to the trace of the associated symmetric observability matrix by sensor-target relative state contribution, $f(\mathcal{W}) = \text{trace}(\mathbb{O}_{\mathcal{W}})$ is modular (submodular) and monotone increasing.
\label{lem_trace_modular}
\end{lem}
\begin{proof}
The symmetric observability matrix by sensor-target relative state contribution associated with $\mathcal{W}$,  $\mathbb{O}_{\mathcal{W}}$ can be calculated as
\begin{equation}
\mathbb{O}_{\mathcal{W}} = R_{wl}^{T}R_{wl} = \sum_{r_{wl}\in \mathcal{W}} r_{wl}^{T}r_{wl}.
\label{eqn:additivity_symmetric}
\end{equation}
Thus, for any $\mathcal{W} \subseteq V$, $\mathbb{O}_{\mathcal{W}}$ is simply a sum of the symmetric observability matrix by sensor-target relative state contribution associated with the individual rows of $R_{\mathcal{W}}$. Given the trace is a linear matrix function, we have 
\begin{eqnarray}
f(\mathcal{W}) &=& \text{trace}(\mathbb{O}_{\mathcal{W}}) \nonumber\\
&=& \text{trace}{\sum_{r_{wl}\in \mathcal{W}} r_{wl}^{T}r_{wl}}\nonumber\\
&=&\sum_{r_{wl}\in \mathcal{W}} \text{trace}(r_{wl}^{T}r_{wl}).
\label{eqn:trace_fun}
\end{eqnarray}
If no sensors are assigned to the target $t_l$, define $trace(\emptyset) =0$. Then we have $f(\mathcal{W}) = \text{trace} (\mathbb{O}_{\mathcal{W}})$ is a modular (submodular) and monotone increasing set function. 
\label{proof_trace_modular}
\end{proof}

\begin{lem}
The set function mapping subsets $\mathcal{W} \subseteq V$ to the rank of the associated symmetric observability matrix by sensor-target relative state contribution, $f(\mathcal{W}) = \text{rank}(\mathbb{O}_{\mathcal{W}})$ is submodular and monotone increasing.
\label{lem_rank_modular}
\end{lem}
\begin{proof}
Given two linear transformations $Q_1, Q_2 \in \mathbb{R}^{n\times n}$, we have 
\begin{align}
&\text{rank}(Q_1+Q_2)\nonumber\\
&=\text{rank}(Q_1)+ \text{rank}(Q_2)\nonumber\\
&-\text{dim}(\text{range}(Q_1) \cap \text{range}(Q_2)).
\end{align}
\label{proof_rank_modular}
From~\cite{lovasz1983submodular}, we know that a set function $f: 2^{\mathcal{V}} \to \mathbb{R}$ is submodular if and only if the derived set functions $f_r : 2^{\mathcal{V}-\{r\}} \to \mathbb{R}$
$$f_r(\mathcal{W}) = f(\mathcal{W}\cup \{r\})-f(\mathcal{W}),$$ are monotone decreasing for all $r\in \mathcal{V}$. We can form the marginal gain functions for rank measure as
\begin{align}
&f_r(\mathcal{W})=\text{rank}(\mathbb{O}_{\mathcal{W}\cup{r}})-\text{rank}(\mathbb{O}_{\mathcal{W}})\nonumber\\
&=\text{rank}(\mathbb{O}_{r})-\text{dim}(\text{range}(\mathbb{O}_{\mathcal{W}}) \cap \text{range}(\mathbb{O}_{r})).
\end{align}
Note that, $\text{rank}(\mathbb{O}_{r})$ is a constant and $\text{dim}(\text{range}(\mathbb{O}_{\mathcal{W}}) )$ only increases with $\mathcal{W}$. Thus, \revo{$f_r$} is monotone decreasing which means $f(\mathcal{W}) = \text{rank}(\mathbb{Q}_{\mathcal{W}})$ is a submodular function. From the additivity property of the $\mathbb{O}_{\mathcal{W}}$ (Equation~\ref{eqn:additivity_symmetric}), it is clear that $f(\mathcal{W}) = \text{rank}(\mathbb{Q}_{\mathcal{W}})$ is monotone increasing. 
\end{proof}

\begin{lem}
The set function mapping subsets $\mathcal{W} \subseteq V$ to the log det of the associated symmetric observability matrix by sensor-target relative state contribution, $f(\mathcal{W}) = \log\det(\mathbb{O}_{\mathcal{W}})$ is submodular and monotone increasing if $\mathbb{O}_{\mathcal{W}}$ is non-singular.
\label{lem_logdet_submodular}
\end{lem}
\begin{proof}
We use the similar idea to show the submodularity of the log det measure, namely, showing that the derived set functions, $f_r : 2^{\mathcal{V}-\{r\}} \to \mathbb{R}$ 
\begin{align*}
&f_r(\mathcal{W})=\log\det(\mathbb{O}_{\mathcal{W}\cup{r}})-\log \det(\mathbb{O}_{\mathcal{W}})\\
&=\log\det(\mathbb{O}_{\mathcal{W}} + \mathbb{O}_{r}) - \log\det(\mathbb{O}_{\mathcal{W}}),
\label{eqn:marginal_logdet}
\end{align*}
are monotone decreasing for any $r\in \mathcal{V}$. Take any $\mathcal{W}_1 \subseteq \mathcal{W}_2 \subseteq \mathcal{V}-\{r\}$. By the additivity property of the $\mathbb{O}_{\mathcal{W}}$ (Equation~\ref{eqn:additivity_symmetric}), it is clear that $\mathcal{W}_1 \subseteq \mathcal{W}_2 \Rightarrow \mathbb{O}_{\mathcal{W}_1} \preceq \mathbb{O}_{\mathcal{W}_2}$. Define $\mathbb{O}(\gamma) = \mathbb{O}_{\mathcal{W}_1} + \gamma(\mathbb{O}_{\mathcal{W}_2} - \mathbb{O}_{\mathcal{W}_1})$ for $\gamma \in [0,1]$. Clearly, $\mathbb{O}(0) = \mathbb{O}_{\mathcal{W}_1}$ and $\mathbb{O}(1) = \mathbb{O}_{\mathcal{W}_2}$. Now define 
\begin{equation}
\hat{f}_r(\mathbb{O}(\gamma)) = \log\det (\mathbb{O}(\gamma) + \mathbb{O}_{r}) -\log\det(\mathbb{O}(\gamma)). 
\label{eqn:auxiliary}
\end{equation}
Note that $\hat{f}_r(\mathbb{O}(0)) = \hat{f}_r(\mathbb{O}_{\mathcal{W}_1})$ and $\hat{f}_r(\mathbb{O}(1)) = \hat{f}_r(\mathbb{O}_{\mathcal{W}_2})$. We have 
\begin{align}
&\dv{}{\gamma} \hat{f}_r(\mathbb{O}(\gamma))\nonumber\\
& = \dv{}{\gamma}[\log\det (\mathbb{O}(\gamma) + \mathbb{O}_{r}) -\log\det(\mathbb{O}(\gamma))]\nonumber\\
& = \text{trace}[(\mathbb{O}(\gamma) + \mathbb{O}_{r})^{-1}(\mathbb{O}_{\mathcal{W}_2}-\mathbb{O}_{\mathcal{W}_1})]\nonumber\\
&~~-\text{trace}[(\mathbb{O}(\gamma))^{-1}(\mathbb{O}_{\mathcal{W}_2}-\mathbb{O}_{\mathcal{W}_1})]\nonumber\\
& = \text{trace}[((\mathbb{O}(\gamma) + \mathbb{O}_{r})^{-1} - (\mathbb{O}(\gamma))^{-1}) (\mathbb{O}_{\mathcal{W}_2}-\mathbb{O}_{\mathcal{W}_1})]\nonumber\\
& \leq 0.
\end{align}
The second equality follows by the matrix derivative formula $\dv{}{t}\log\det X(\gamma) = \text{trace}[X(\gamma)^{-1}\dv{}{\gamma}X(\gamma)]$~\cite{petersen2008matrix}. \revo{Notably, since $\mathbb{O}(\gamma)\succeq 0$ and $\mathbb{O}_{r} \succeq 0$, we have $\mathbb{O}(\gamma) +  \mathbb{O}_{r} \succeq \mathbb{O}(\gamma)$ and thus $(\mathbb{O}(\gamma) + \mathbb{O}_{r})^{-1} - (\mathbb{O}(\gamma))^{-1} \preceq 0$.  Given} $(\mathbb{O}(\gamma) + \mathbb{O}_{r})^{-1} - (\mathbb{O}(\gamma))^{-1} \preceq 0$, and $\mathbb{O}_{\mathcal{W}_2}-\mathbb{O}_{\mathcal{W}_1}\succeq 0$, the last inequality holds. Since 
$$\hat{f}_r(\mathbb{O}(1)) = \hat{f}_r(\mathbb{O}(0)) + \int_{0}^{1} \dv{}{\gamma} \hat{f}_r(\mathbb{O}(\gamma)) d\gamma,$$ it follows that $\hat{f}_r(\mathbb{O}(1)) = \hat{f}_r(\mathbb{O}_{\mathcal{W}_2}) \leq \hat{f}_r(\mathbb{O}(0)) = \hat{f}_r(\mathbb{O}_{\mathcal{W}_1})$. Thus, we have $f_r$ is monotone decreasing, and $f$ is submodular. Similarly, the additivity property of the $\mathbb{O}_{\mathcal{W}}$ (Equation~\ref{eqn:additivity_symmetric}) shows that $f$ is monotone increasing. 

However, the proof of the monotonicity and submodularity for the log det is based on the condition that $\mathbb{O}(p_{t_l})$ is non-singular. Otherwise, the conclusion does not hold. For example, if a single sensor is assigned to target $t_l$, $\mathbb{O}(p_{t_l})$ is always singular (see the proof of Corollary~\ref{cor:sensor_n_1}), and thus $\log\det \mathbb{O}(p_{t_l}) = -\infty$. If no sensors are  assigned, we define the empty set case as $\log\det(\emptyset) = 0$. Then, the function is not monotone increasing. 

The proof of the monotonicity and submodularity for the trace, rank and log det of the \emph{symmetric observability matrix}, $\mathbb{O}(p_{t_l},u_{t_l}):=O^{T}(p_{t_l},u_{t_l})O(p_{t_l},u_{t_l})$ is similar to the proof for that of the \emph{symmetric observability matrix by the sensor-target relative state contribution}, $\mathbb{O}(p_{t_l})$ as provided above. Because, from Equation~\ref{sy_ob_matrix}, we know, 
\begin{equation}
  \mathbb{O}(p_{t_l},u_{t_l}) =
\mathbb{O}(p_{t_l})+O^{T}(u_{t_l})O(u_{t_l}).  
\end{equation}
The \emph{unknown control input contribution} $O^{T}(u_{t_l})O(u_{t_l}) \succeq 0$ does not affect the properties of the measures for the \emph{symmetric observability matrix}, $\mathbb{O}(p_{t_l},u_{t_l})$, since assigning sensors only influences the \emph{sensor-target relative state contribution part}, $\mathbb{O}(p_{t_l})$ . But for the log det of \emph{symmetric observability matrix}, we need to guarantee  $\mathbb{O}(p_{t_l},u_{t_l})$ is non-singular.
\label{proof_log_det}
\end{proof}
}

\newpage
\bibliographystyle{IEEEtran}
\bibliography{refs.bib}

\begin{IEEEbiography}[{\includegraphics[width=1in,height=1.25in,clip,keepaspectratio]{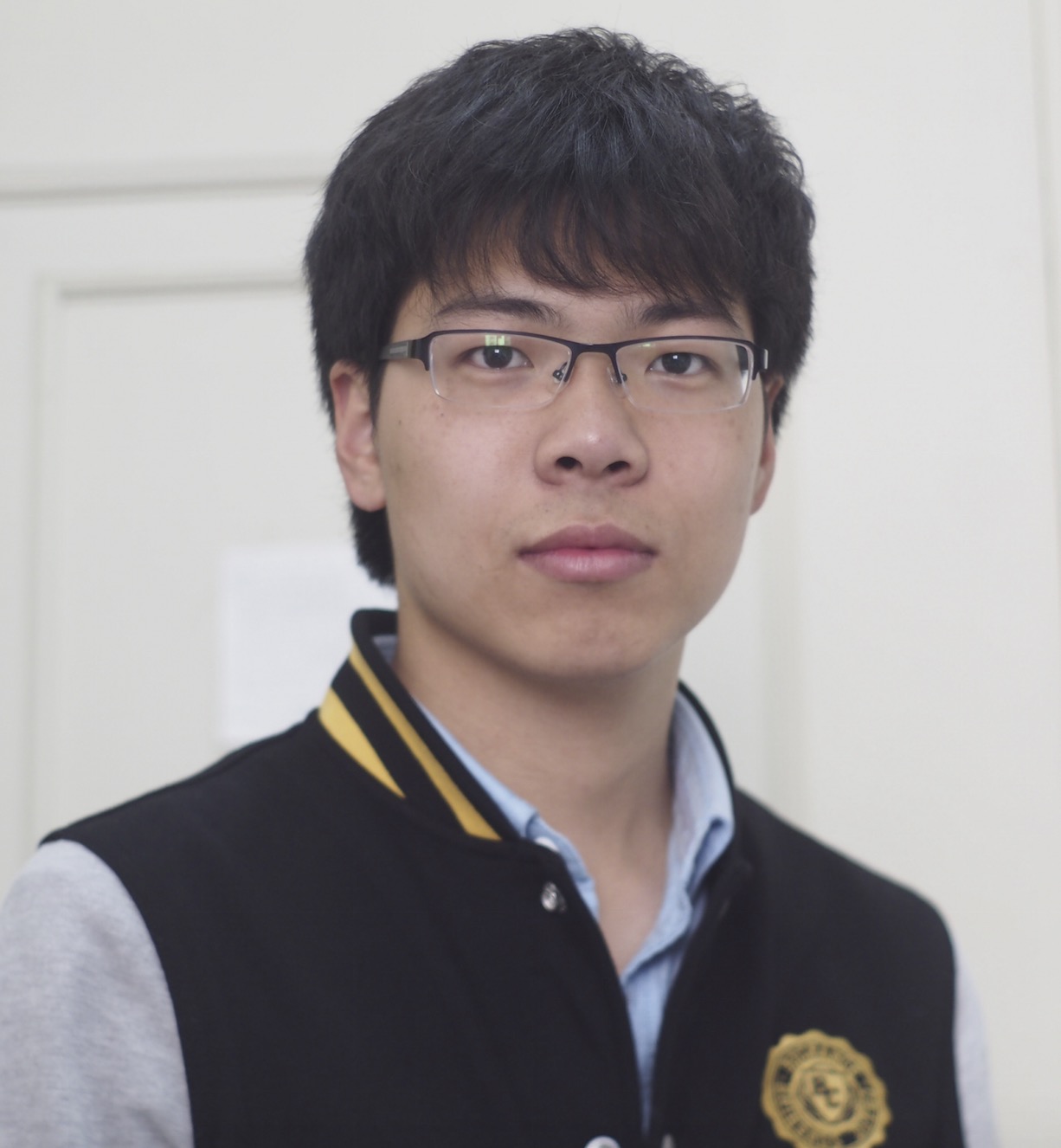}}]{Lifeng Zhou}  received the
B.S. degree in Automation from Huazhong University of Science and Technology, Wuhan, China, in 2013, the M.Sc.
degree in Automation from Shanghai Jiao Tong University, Shanghai, China, in 2016. He
is currently pursuing the Ph.D. degree in Electrical and Computer Engineering, Virginia Tech,
Blacksburg, VA, USA. 

His research interests include multi-robot coordination, event-based control, sensor assignment and risk-averse decision making.
\end{IEEEbiography}

% if you will not have a photo at all:
\begin{IEEEbiography}[{\includegraphics[width=1in,height=1.25in,clip,keepaspectratio]{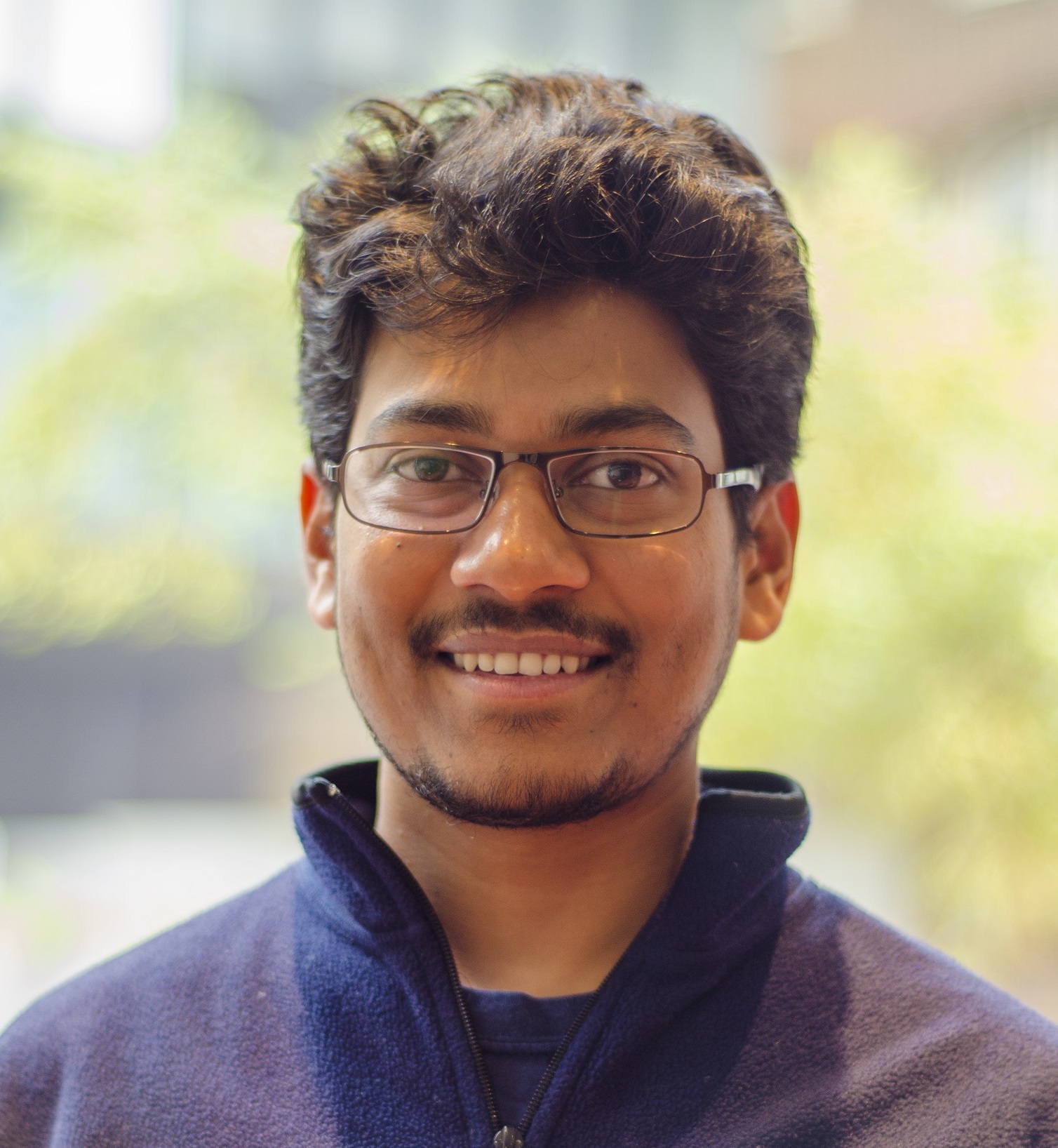}}]{Pratap Tokekar} is an Assistant Professor in the Department of Electrical and Computer Engineering at Virginia Tech. Previously, he was a Postdoctoral Researcher at the GRASP lab of University of Pennsylvania. He obtained his Ph.D. in Computer Science from the University of Minnesota in 2014 and Bachelor of Technology degree in Electronics and Telecommunication from College of Engineering Pune, India in 2008. He is a recipient of the NSF CISE Research Initiation Initiative award. His research interests include algorithmic and field robotics, and cyber physical systems, and their applications to precision agriculture and environmental monitoring.
\end{IEEEbiography}

\end{document}